\documentclass{article}
\usepackage{microtype}
\usepackage{graphicx}
\usepackage{subfigure}
\usepackage{booktabs} 
\usepackage{colortbl}
\definecolor{light-gray}{gray}{0.90}
\usepackage{hyperref}

\usepackage[accepted]{icml2022}
\usepackage{amsmath}
\usepackage{amssymb}
\usepackage{mathtools}
\usepackage{amsthm}
\usepackage{wrapfig}
\usepackage{times}
\usepackage{helvet}
\usepackage{amsmath}
\usepackage{courier}
\usepackage{times}
\usepackage{amssymb}
\usepackage{booktabs}
\usepackage{graphicx} 
\usepackage{float} 
\usepackage{subfigure}
\usepackage[capitalize,noabbrev]{cleveref}
\theoremstyle{plain}
\newtheorem{theorem}{Theorem}[section]

\theoremstyle{definition}

\theoremstyle{remark}

\usepackage[textsize=tiny]{todonotes}
\icmltitlerunning{ProGCL: Rethinking Hard Negative Mining in Graph Contrastive Learning}
\begin{document}
\twocolumn[
\icmltitle{ProGCL: Rethinking Hard Negative Mining in Graph Contrastive Learning}
\begin{icmlauthorlist}
\icmlauthor{Jun Xia}{wlu,wias}
\icmlauthor{Lirong Wu}{wlu,wias}
\icmlauthor{Ge Wang}{wlu,wias}
\icmlauthor{Jintao Chen}{zju}
\icmlauthor{Stan Z. Li}{wlu,wias}
\end{icmlauthorlist}
\icmlaffiliation{wlu}{Westlake University}
\icmlaffiliation{wias}{Westlake Institute for Advanced Study}
\icmlaffiliation{zju}{School of Computer Science, Zhejiang University}
\icmlcorrespondingauthor{Jun Xia}{xiajun@westlake.edu.cn}
\icmlkeywords{Machine Learning, ICML}
\vskip 0.3in
]
\printAffiliationsAndNotice{}
\begin{abstract}
Contrastive Learning (CL) has emerged as a dominant technique for unsupervised representation learning which embeds augmented versions of the anchor close to each other (positive samples) and pushes the embeddings of other samples (negatives) apart. As revealed in recent studies, CL can benefit from hard negatives (negatives that are most similar to the anchor). However, we observe limited benefits when we adopt existing hard negative mining techniques of other domains in Graph Contrastive Learning (GCL). We perform both experimental and theoretical analysis on this phenomenon and find it can be attributed to the message passing of Graph Neural Networks (GNNs). Unlike CL in other domains, most hard negatives are potentially false negatives (negatives that share the same class with the anchor) if they are selected merely according to the similarities between anchor and themselves, which will undesirably push away the samples of the same class. To remedy this deficiency, we propose an effective method, dubbed \textbf{ProGCL}, to estimate the probability of a negative being true one, which constitutes a more suitable measure for negatives' hardness together with similarity. Additionally, we devise two schemes (i.e., \textbf{ProGCL-weight} and \textbf{ProGCL-mix}) to boost the performance of GCL. Extensive experiments demonstrate that ProGCL brings notable and consistent improvements over base GCL methods and yields multiple state-of-the-art results on several unsupervised benchmarks or even exceeds the performance of supervised ones. Also, ProGCL is readily pluggable into various negatives-based GCL methods for performance improvement. We release the code at \textcolor{magenta}{\url{https://github.com/junxia97/ProGCL}}.
\end{abstract}
\begin{figure}[ht]
    \subfigure[CIFAR-10 (Image)]{
    \label{fig1-a}
    \includegraphics[width=0.233\textwidth]{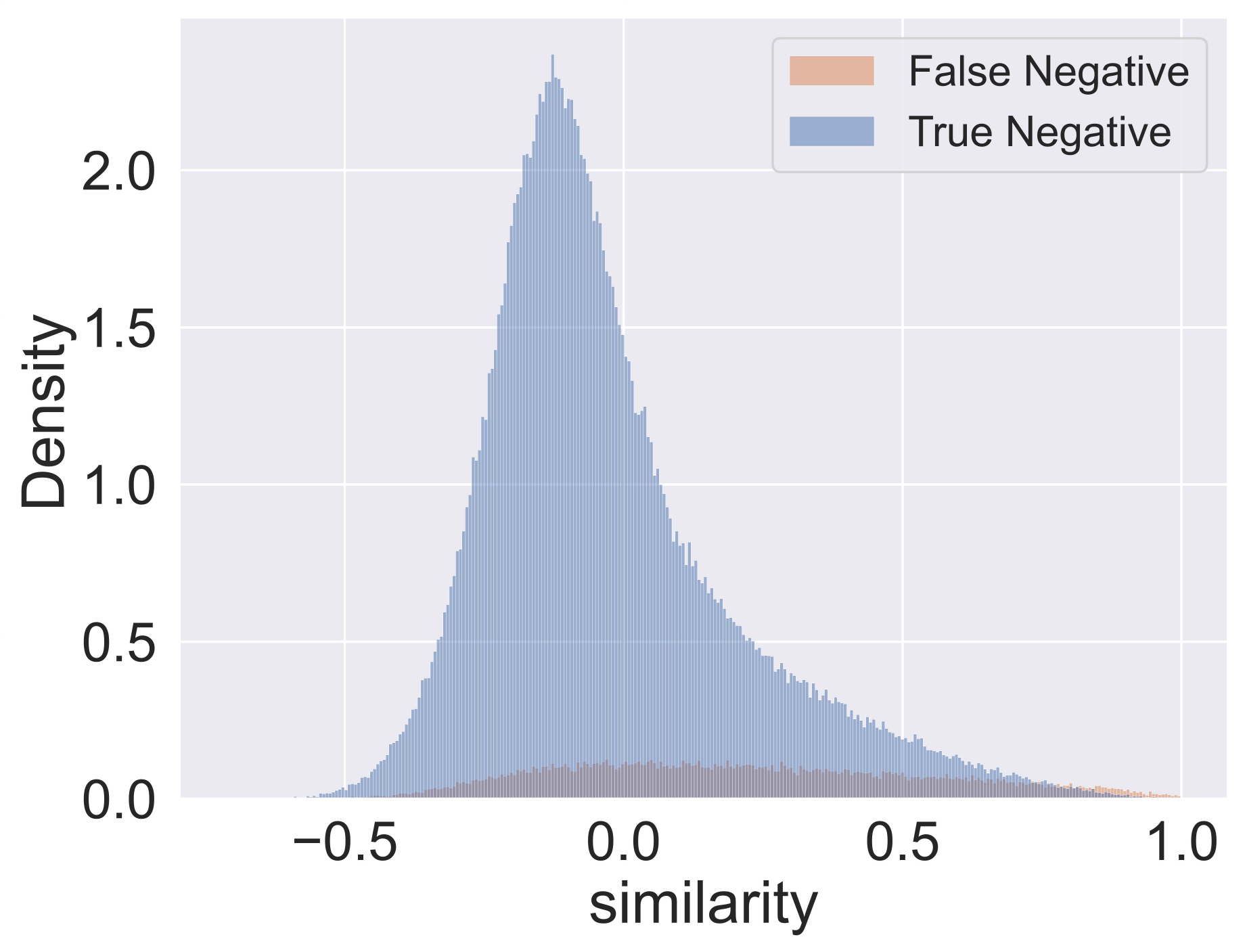}}
    \hspace{-3mm}
    \subfigure[Coauthor-CS (Graph)]{
    \label{fig1-b}
    \includegraphics[width=0.233\textwidth]{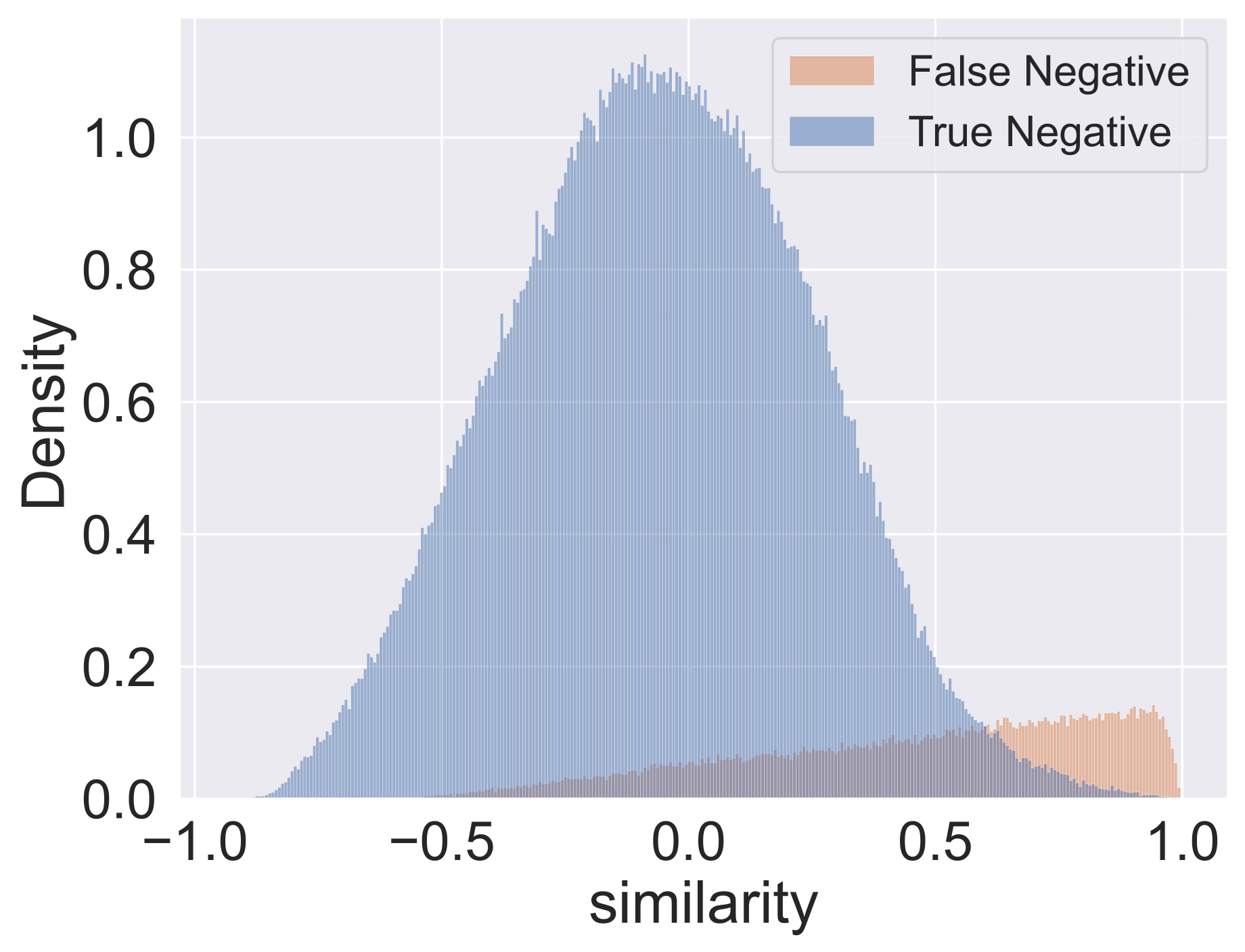}}
    \vspace{-3.6pt}
    \caption{Similarity (cosine similarity between normalized embeddings of anchor and negatives) histograms of negatives. We adopt SimCLR~\cite{chen2020simple} and GCA~\cite{zhu2021graph} for CIFAR-10 and Coauthor-CS respectively. False (or true) negatives denote the samples that are (or not) from the same class with the anchor. Unlike CL, most negatives with larger similarities to the anchor are false ones in GCL. More examples are in the appendix.}
    \label{fig_first}
    \vspace{-3.6pt}
\end{figure}
\section{Introduction}
\label{intro}
Recently, Contrastive Learning (CL) has demonstrated unprecedented unsupervised performance in computer vision \cite{he2020momentum,chen2020simple}, natural language processing~\cite{gao2021simcse,zheng-etal-2022-using} and graph representation learning~\cite{velickovic2019deep,10.1145/3485447.3512156,zhu2021graph}. As with metric learning~\cite{kaya2019deep}, existing works both theoretically and practically validate that hard negatives contribute more to CL. For example, HCL~\cite{robinson2021contrastive} develops a family of unsupervised sampling methods for selecting hard negatives and MoCHi~\cite{kalantidis2020hard} mixes the hard negatives to synthesize more hard negative ones. However, we observe minor improvement or even significant performance drop when we adopt these negative mining techniques in GCL (the results can be seen in Table~\ref{table5}). Concurrent to our work, \citet{zhu2021an} also observe that existing hard negative mining techniques that work well in other domains bring limited benefits to GCL. Unfortunately, they didn't provide any solution to tackle this issue.

To explain these phenomena, we first plot the negatives' distributions over similarity of various datasets in Figure~\ref{fig_first}. Note that we do not observe significant changes of negatives' distribution (keep unimodal as shown in Figure~\ref{fig1-a}) during the training process of SimCLR~\cite{chen2020simple} on CIFAR-10 and other image datasets. However, for the first stage of GCL, the negatives' distribution is bimodal for a long period as Figure~\ref{fig1-b} and then progressively transit to unimodal distribution as CL at the second stage. Similar phenomena for more datasets can be found in the appendix. The difference of negatives' distribution between CL and GCL can be attributed to the unique message passing of Graph Neural Networks, which we further discuss in section~\ref{analysis}. These provide explanations for the poor performance of existing negative mining techniques in GCL. Specifically, they regard the negatives that are most similar to anchor points as hard ones across all the training process. However, as shown in Figure~\ref{fig1-b}, most selected ``hard'' negatives in this way are false negatives for GCL indeed, which will undesirably push away the semantically similar samples and thus degrade the performance. The existence of false negatives is termed as \emph{sampling bias} in DCL~\cite{chuang2020debiased}. However, as reported in Table~\ref{table5} and \citet{zhu2021an}, DCL brings performance drop for GCL because the sampling bias is severer in GCL. Now, we are naturally motivated to ask following questions: \emph{Are there better alternatives to measure negatives' hardness considering the issue of false negatives in GCL? Can we devise more suitable methods to eliminate severer bias in GCL?}

To answer these questions, we argue that true and false negatives can be distinguished by fitting a two-component (true-false) beta mixture model (BMM) on the similarity. The posterior probability of a negative being true one under BMM can constitute a more suitable measure for negatives' hardness accompanied with similarity. With the novel measure, we devise two schemes (ProGCL-weight and ProGCL-mix) for further improvement of negatives-based GCL methods. To the best of our knowledge, our work makes one of the pioneering attempts to study hard negative mining in \emph{node-level} GCL. We highlight the following contributions:
\vspace{-2.8pt}
\begin{itemize}
\item We demonstrate the difference of negatives' distribution between GCL and CL and explain why existing hard negative mining techniques can not work well in GCL with both theoretical and experimental analysis.
\item We propose to utilize BMM to estimate the probability of a negative being true one relative to a specific anchor. Combined with the similarity, we obtain a more suitable measure for negatives' hardness.
\item We devise two schemes (i.e., ProGCL-weight and ProGCL-mix) that are more suitable for hard negative mining in GCL.
\item ProGCL brings notable and consistent improvements over base GCL methods and yields multiple state-of-the-art results on several unsupervised benchmarks or even exceeds the performance of supervised ones. Also, it can boost various negatives-based GCL methods for further improvements.
\end{itemize}

\section{Related Work}
\subsection{Graph Contrastive Learning (GCL)}
Recently, GCL has gained popularity in unsupervised graph representation learning, which can get rid of the resource-intensive annotations~\cite{xia2021towards,tan2021co,9747279}. Initially, DGI~\cite{velickovic2019deep} and InfoGraph~\cite{Sun2020InfoGraph:} are proposed to obtain expressive representations for graphs or nodes via maximizing the mutual information between graph-level representations and substructure-level representations of different granularity. Similarly, GMI~\cite{peng2020graph} adopts two discriminators to directly measure mutual information between input and representations of both nodes and edges. Additionally, MVGRL~\cite{hassani2020contrastive} proposes to learn both node-level and graph-level representation by performing node diffusion and contrasting node representation to augmented graph representation. Different from our work, GraphCL~\cite{You2020GraphCL} and its variants~\cite{suresh2021adversarial,xu2021infogcl,you2021graph,you2022bringing,li2021disentangled} adopt SimCLR framework and design various augmentations for graph-level representation learning. For node-level representations, GRACE~\cite{zhu2020deep} and its variants~\cite{zhu2021graph,tong2021directed} maximize the agreement of node embeddings across two corrupted views of the graph. More recently, SimGRACE~\cite{10.1145/3485447.3512156}, BGRL~\cite{thakoor2021bootstrapped} and CCA-SSG~\cite{zhang2021canonical} try to simplify graph contrastive learning via discarding the negatives, parameterized mutual information estimator or even data augmentations. In this paper, we consider hard negatives mining to further boost GCL on \emph{node-level representation learning}, which is still under-explored. For the limited space, we recommend readers refer to a recent review~\cite{xia2022survey} for more relevant literatures.
\subsection{Hard Negative Mining in Contrastive Learning}
 For contrastive learning, HCL~\cite{robinson2021contrastive} and MoCHi~\cite{kalantidis2020hard} are proposed to emphasize or synthetize hard negatives. Additionally, Ring~\cite{wu2021conditional} introduces a family of mutual information estimators that sample negatives in a “ring” around each positive.  However, as shown in Table~\ref{table5}, these techniques which emphasize hard negatives don't work well in GCL. For GCL, Zhao et al.~\cite{zhao2021graph} utilize the clustering pseudo labels to alleviate false negative issue, which suffers from heavy computational overhead and will degrade the performance when confronted with multi-class datasets. CuCo~\cite{ijcai2021-317} sorts the negatives from easy to hard with similarity for graph-level contrastive learning and proposes to automatically select and train negative samples with curriculum learning. Instead, we study node-level GCL with message passing among instances. Additionally, Zhu et al.~\cite{Zhu:2022se} propose a structure-enhanced negative mining scheme which discovers hard negatives in each metapath-induced view in heterogeneous graph. However, it is not suitable for homogeneous graphs without metapath.
\section{Methodology}
\begin{figure}[ht]
    \begin{center}
    \includegraphics[width=0.45\textwidth]{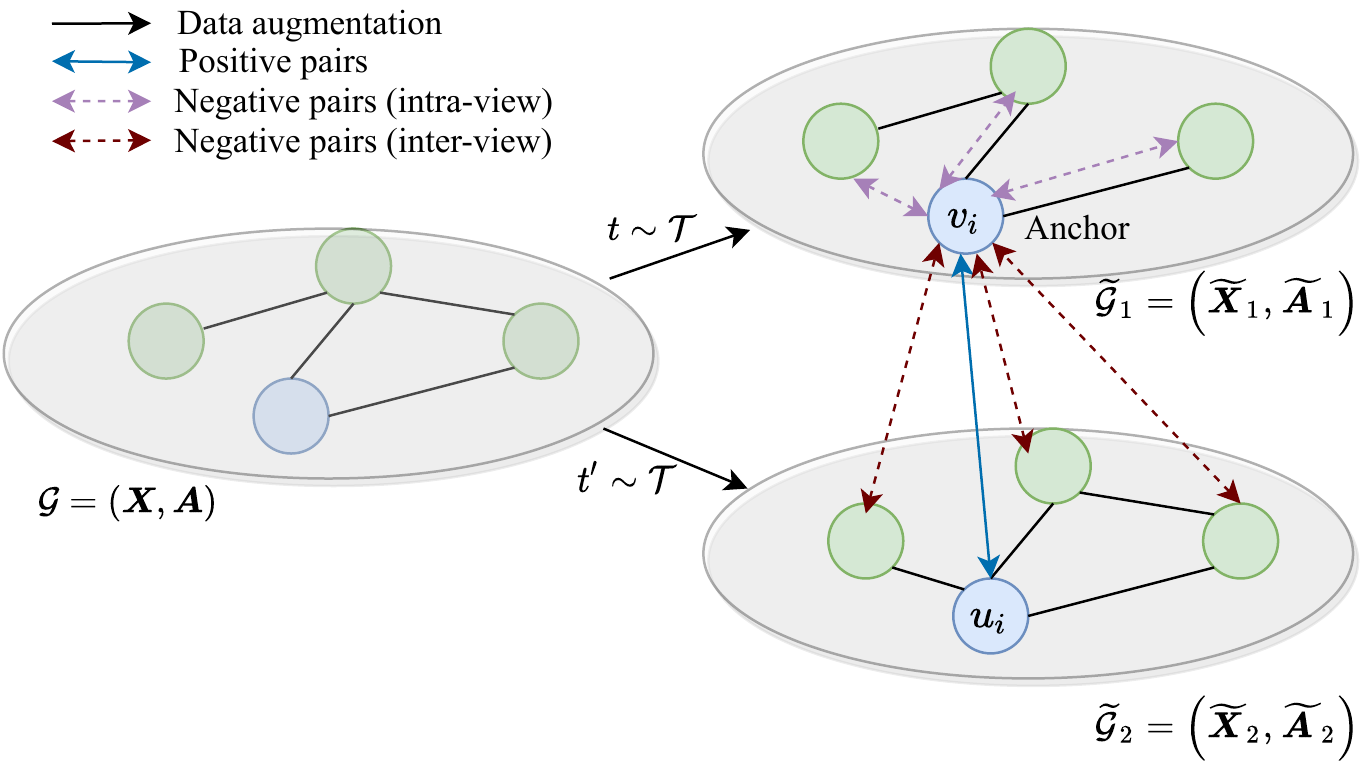}
    \end{center}
    \caption{Schematic diagram of node-level GCL framework.}
    \label{fig_2}
\end{figure}
\subsection{Preliminaries and Notations}
Let $\mathcal{G}$ = $(\mathcal{V},\mathcal{E})$ be the graph, where $\mathcal{V} = \left\{v_{1}, v_{2}, \cdots, v_{N}\right\}, \mathcal{E} \subseteq \mathcal{V} \times \mathcal{V}$ denote the node set and edge set respectively. Additionally, $\boldsymbol{X} \in \mathbb{R}^{N \times F}$ and $\boldsymbol{A} \in\{0,1\}^{N \times N}$ are the feature matrix and the adjacency matrix. $\boldsymbol{f}_{i} \in \mathbb{R}^{F}$ is the feature of $v_i$, and $\boldsymbol{A}_{i j}=1$ iff $\left(v_{i}, v_{j}\right) \in \mathcal{E}$. Our objective is to learn a GNN encoder $f(\boldsymbol{X}, \boldsymbol{A}) \in \mathbb{R}^{N \times F^{\prime}}$ to embed nodes in low dimensional space without label information. These low dimensional representations can be used in downstream tasks including node classification. As shown in Figure~\ref{fig_2},  we sample two augmentation functions $t \sim \mathcal{T}$ and $t^{\prime} \sim \mathcal{T}$, here $\mathcal{T}$ is the set of all augmentation functions. We can then obtain two views for $\mathcal{G}$,  $\widetilde{\mathcal{G}}_{1}=t(\mathcal{G})$ and $\widetilde{\mathcal{G}}_{2}=t^{\prime}(\mathcal{G})$. Given  $\widetilde{\mathcal{G}}_{1}=(\widetilde{\boldsymbol{X}}_{1}, \widetilde{\boldsymbol{A}}_{1})$ and $\widetilde{\mathcal{G}}_{2}=(\widetilde{\boldsymbol{X}}_{2}, \widetilde{\boldsymbol{A}}_{2})$, we denote node embeddings in the two views as $\boldsymbol{U}=f(\widetilde{\boldsymbol{X}}_{1}, \widetilde{\boldsymbol{A}}_{1})$ and $\boldsymbol{V}=f(\widetilde{\boldsymbol{X}}_{2}, \widetilde{\boldsymbol{A}}_{2})$. For any node $v_i$, its embedding in one view $\boldsymbol{v}_{i}$ is regarded as the anchor. The embedding  $\boldsymbol{u_i}$ in the other view is the positive sample and the embeddings of other nodes in both views are negatives. Similar to the InfoNCE~\cite{oord2018representation}, the training objective for each positive pair $\left(\boldsymbol{u}_{i}, \boldsymbol{v}_{i}\right)$ is,
\begin{equation}
\label{eq1}
\begin{aligned}
    \ell&\left(\boldsymbol{u}_{i},\boldsymbol{v}_{i}\right)=\\
    &\log \frac{e^{\theta\left(\boldsymbol{u}_{i}, \boldsymbol{v}_{i}\right) / \tau}}{\underbrace{e^{\theta\left(\boldsymbol{u}_{i}, \boldsymbol{v}_{i}\right) / \tau}}_{\text{positive pair }}+\underbrace{\sum_{k\neq i}e^{\theta\left(\boldsymbol{u}_{i}, \boldsymbol{v}_{k}\right) / \tau}}_{\text{inter-view negative pairs}}+\underbrace{\sum_{k\neq i}e^{\theta\left(\boldsymbol{u}_{i}, \boldsymbol{u}_{k}\right) / \tau}}_{\text{intra-view negative pairs}}},
\end{aligned}
\end{equation}
where the critic $\theta(\boldsymbol{u}, \boldsymbol{v})= s(g(\boldsymbol{u}), g(\boldsymbol{v}))$. Here $s(\cdot, \cdot)$ is the cos similarity and $g(\cdot)$ is linear projection (two-layer perceptron model) to enhance the expression power of the critic function~\cite{chen2020simple}. With the symmetry of the two views, we can then define the overall loss as the average of all the positive pairs,
\begin{equation}
\label{eq2}
\mathcal{J}=-\frac{1}{2 N} \sum_{i=1}^{N}\left[\ell\left(\boldsymbol{u_{i}}, \boldsymbol{v_{i}}\right)+\ell\left(\boldsymbol{v_{i}}, \boldsymbol{u_{i}}\right)\right].
\end{equation}
Many GCL methods are built on this framework~\cite{zhu2021an,qiu2020gcc,zhu2021graph,Jin2021MultiScaleCS}, which motivates us to study hard negative mining on it.

\subsection{Experimental and Theoretical Analysis}
\label{analysis}
\begin{figure}[ht]
    \subfigure[GCN]{
    \label{figany1}
    \includegraphics[width=0.23\textwidth]{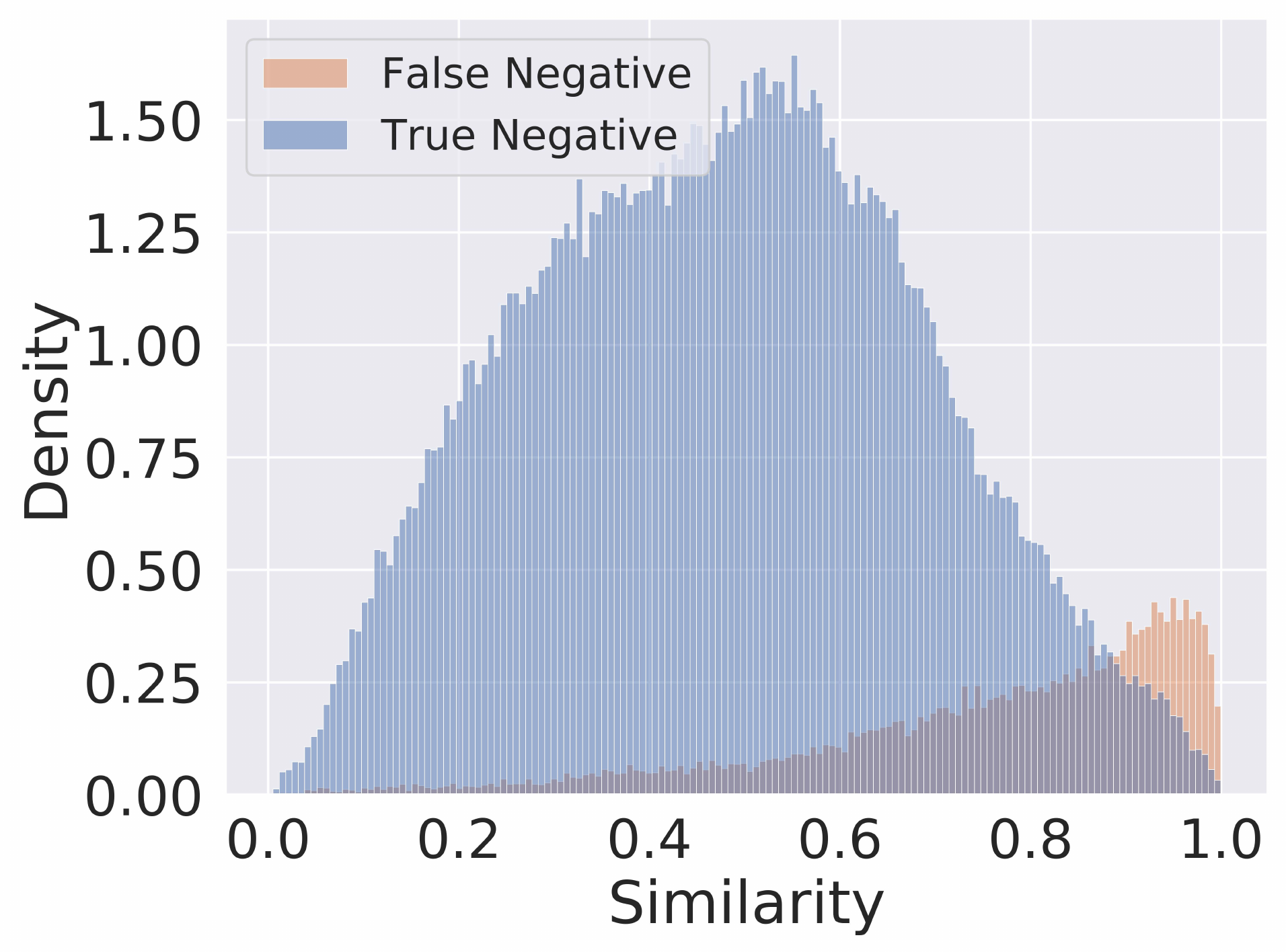}}
    \hspace{-3mm}
    \subfigure[MLP]{
    \label{figany2}
    \includegraphics[width=0.23\textwidth]{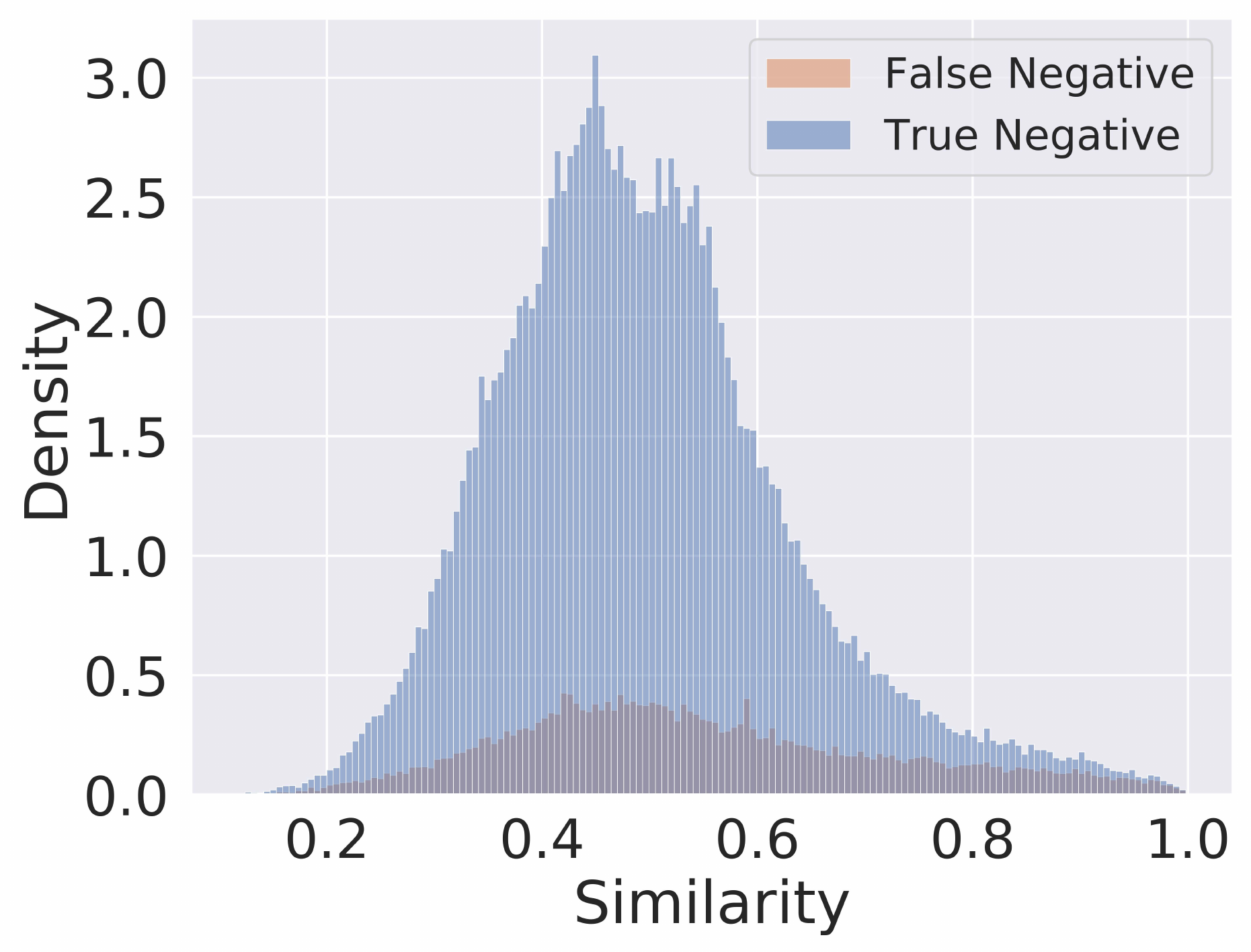}}
    \caption{Similarity (after Min-Max normalization) histograms of GCA~\cite{zhu2021graph} with GCN (with message passing) or MLP (without message passing) as encoder on Amazon-Photo.}
    \label{figanysis}
\end{figure}
In Figure~\ref{figanysis}, we study the message passing's role in GCL via replacing 2-layer GCN~\cite{kipf2016semi} that is composed of message passing and multi-layer perceptron (MLP) with a 2-layer MLP encoder only. As can be observed, the similarity histograms of GCL will be similar to that of CL without message passing, which verifies that message passing of GNN encoder is the key factor for the difference of negatives' distribution between CL and GCL. More specifically, for the first bimodal stage of GCL, the message passing in GCL enlarges the similarities between neighbour nodes which often share the same class. For the second unimodal stage, instance discrimination of GCL occupies a position of prominence and pushes away all the other samples regardless of their semantic class. Theoretically, for embeddings of any nodes pair in the graph $\mathcal{G}$, we can compare their distance before and after message passing and show that the distance will be decreased with the process, as formally induced in Theorem~\ref{theorem}.
\begin{theorem}
\label{theorem}
$\mathcal{G}$ is a non-bipartile and connected graph with $N$ nodes $\mathcal{V}=\{v_1, \dots,v_N\}$, and $\boldsymbol{X}_i^{(\tau)}$ is the embedding of node $v_i$ after $\tau$ times message passing ($\boldsymbol{X}_i^{(0)}=\boldsymbol{f}_i$). For large enough $\tau$,  $||\boldsymbol{X}_i^{(\tau)}-\boldsymbol{X}_j^{(\tau)}||_2 \le ||\boldsymbol{X}_i^{(0)}-\boldsymbol{X}_j^{(0)}||_2$.
\end{theorem}
\begin{proof}
Given the message passing is, 
$$
\boldsymbol{X}^{(t+1)}=\boldsymbol{\hat{D}}^{-\frac{1}{2}} \boldsymbol{\hat{A}} \boldsymbol{\hat{D}}^{-\frac{1}{2}}\boldsymbol{X}^{(t)},
$$
where  $\boldsymbol{\hat{A}}=\boldsymbol{A}+\boldsymbol{I}$ and $\boldsymbol{\hat{D}}_i=\sum_{j}\boldsymbol{\hat{A}}_{ij}$. We can then rewrite the message passing as,
$$
\boldsymbol{X}^{(t+1)}=(\boldsymbol{I}-\boldsymbol{L}_{sym})\boldsymbol{X}^{(t)},
$$
where $\boldsymbol{L}_{sym}=\boldsymbol{\hat{D}}^{-\frac{1}{2}} \boldsymbol{\hat{L}} \boldsymbol{\hat{D}}^{-\frac{1}{2}}$, $\boldsymbol{\hat{L}}=\boldsymbol{\hat{D}}-\boldsymbol{\hat{A}}$.
Let $(\lambda_1, \dots, \lambda_N)$ and $(\boldsymbol{e}_1, \dots, \boldsymbol{e}_N)$ denote the eigenvalue and eigenvector of matrix $\boldsymbol{I}-\boldsymbol{L}_{sym}$, respectively. With the property of symmetric Laplacian matrix for non-bipartile and connected graph,
$$
-1 < \lambda_1 < \lambda_2 < \cdots < \lambda_N=1, \quad
\boldsymbol{e}_N=\boldsymbol{\hat{D}}^{-\frac{1}{2}}[1, 1, \dots, 1]^{T},
$$
we can rewrite $\boldsymbol{X}_i^{(\tau)}-\boldsymbol{X}_j^{(\tau)}$ as
$$
\begin{aligned}
\boldsymbol{X}_i^{(\tau)}-\boldsymbol{X}_j^{(\tau)}=[(\boldsymbol{I}-\boldsymbol{L}_{sym})^{\tau}\boldsymbol{X}]_i-[(\boldsymbol{I}-\boldsymbol{L}_{sym})^{\tau}\boldsymbol{X}]_j \\
= [\lambda_1^{\tau}(\boldsymbol{e}_{1}^{(i)}-\boldsymbol{e}_{1}^{(j)}),\dots, \lambda_{n-1}^{\tau}(\boldsymbol{e}_{N-1}^{(i)}-\boldsymbol{e}_{N-1}^{(j)}),0]\boldsymbol{\hat{X}}
\end{aligned}
$$
$\boldsymbol{e}_{k}^{(i)}$ is the $i^{th}$ element of eigenvector $\boldsymbol{e}_{k}$, $\boldsymbol{\hat{X}}$ is the coordinate matrix of $\boldsymbol{X}$ in the space spanned by eigenvectors $(\boldsymbol{e}_1, \dots, \boldsymbol{e}_N)$.
Thus,
$$
||\boldsymbol{X}_i^{(\tau)}-\boldsymbol{X}_j^{(\tau)}||_2= \sqrt{\sum_{m=1}^{N}[\sum_{k=1}^{N-1}\lambda_k^{\tau}(\boldsymbol{e}_{k}^{(i)}-\boldsymbol{e}_{k}^{(j)})\boldsymbol{\hat{X}}_{km}]^2}
$$
Because $1 < \lambda_1 < \lambda_2 < \cdots < \lambda_{N-1} < 1$, thus for a large $\tau$,  $||\boldsymbol{X}_i^{(\tau)}-\boldsymbol{X}_j^{(\tau)}||_2 \le ||\boldsymbol{X}_i^{(0)}-\boldsymbol{X}_j^{(0)}||_2$.
\end{proof}
\subsection{ProGCL}
We aim to estimate the probability of a negative being true one. As can be observed in Figure~\ref{fig_3}, there is a significant difference between the false negatives and true negatives' distributions in GCL, allowing the two types of negatives can be distinguished from the similarity distribution. Here we propose to utilize mixture model to estimate the probability. Gaussian Mixture Model (GMM) is the most popular mixture model~\cite{lindsay1995mixture}. However, in Figure~\ref{fig_3}, the distribution of false negatives is skew and thus symmetric Gaussian distribution can not fit this well. To circumvent this issue, we resort to beta distribution~\cite{gupta2004handbook,ji2005applications} which is flexible enough to model various distributions (symmetric, skewed, arched distributions and so on) over $[0,1]$. As can be observed in Figure~\ref{fig_3}, Beta Mixture Model (BMM) can fit the empirical distribution better than GMM. Also, we compare the performance of ProGCL with BMM and GMM in Table~\ref{table4} and find that BMM consistently outperforms GMM. The probability density function (pdf) of beta distribution is,
\begin{equation}
\label{eq4}
p(s \mid \alpha, \beta)=\frac{\Gamma(\alpha+\beta)}{\Gamma(\alpha) \Gamma(\beta)} s^{\alpha-1}(1-s)^{\beta-1},
\end{equation}
where $\alpha, \beta > 0$ are the parameters of beta distribution and  $\Gamma(\cdot)$ is gamma function. The pdf of beta mixture model of $C$ components on $s$ (Min-Max normalized cosine similarity between normalized embeddings of anchors and negatives) can be defined as:
\begin{equation}
\label{eq3}
p(s)=\sum_{c=1}^{C} \lambda_{c} p(s \mid \alpha_c, \beta_c),
\end{equation}
\begin{figure}[t]
    \subfigure[Amazon-Photo]{
    \label{fig3-a}
    \includegraphics[width=0.23\textwidth]{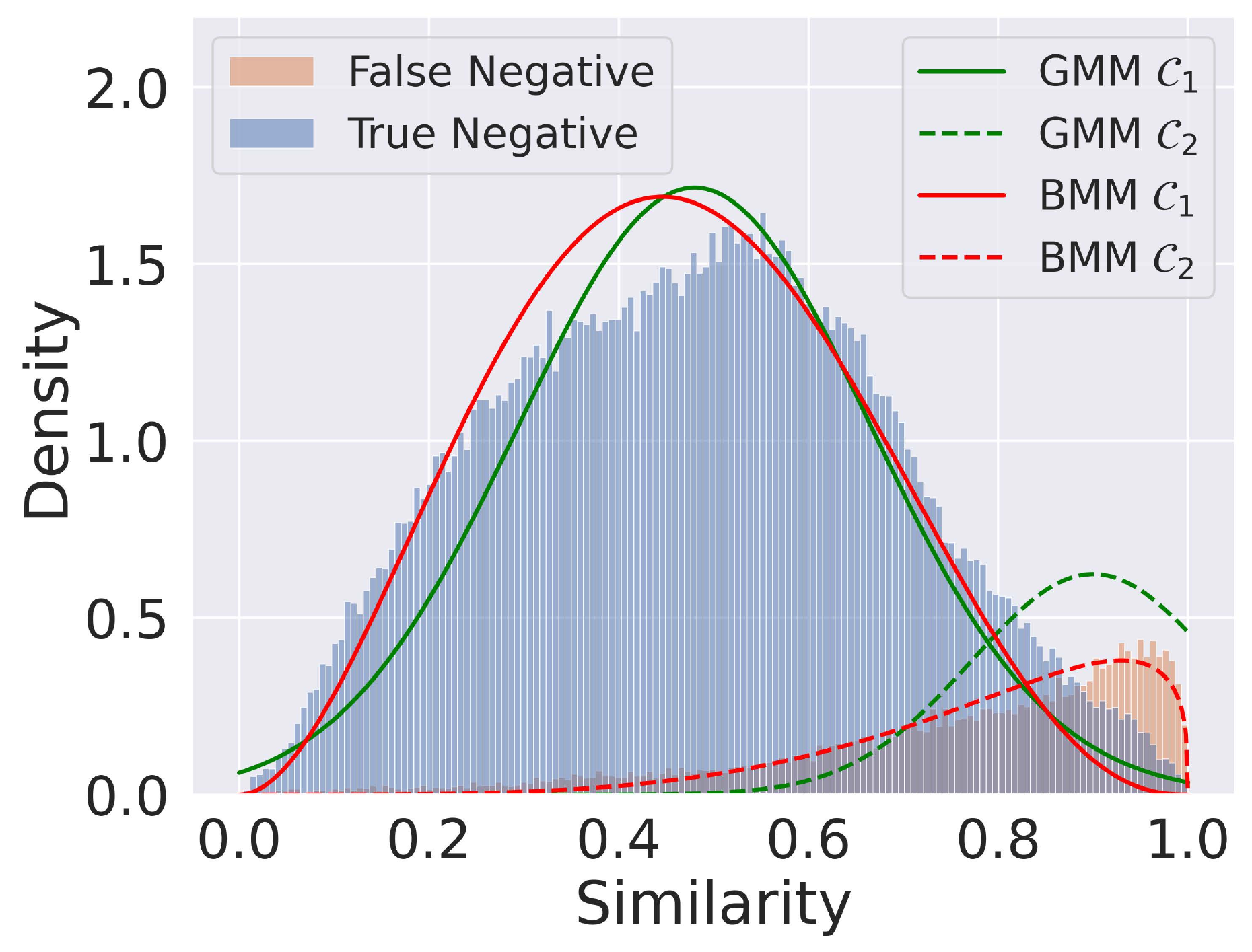}}
    \hspace{-3mm}
    \subfigure[Coauthor-CS]{
    \label{fig3-b}
    \includegraphics[width=0.23\textwidth]{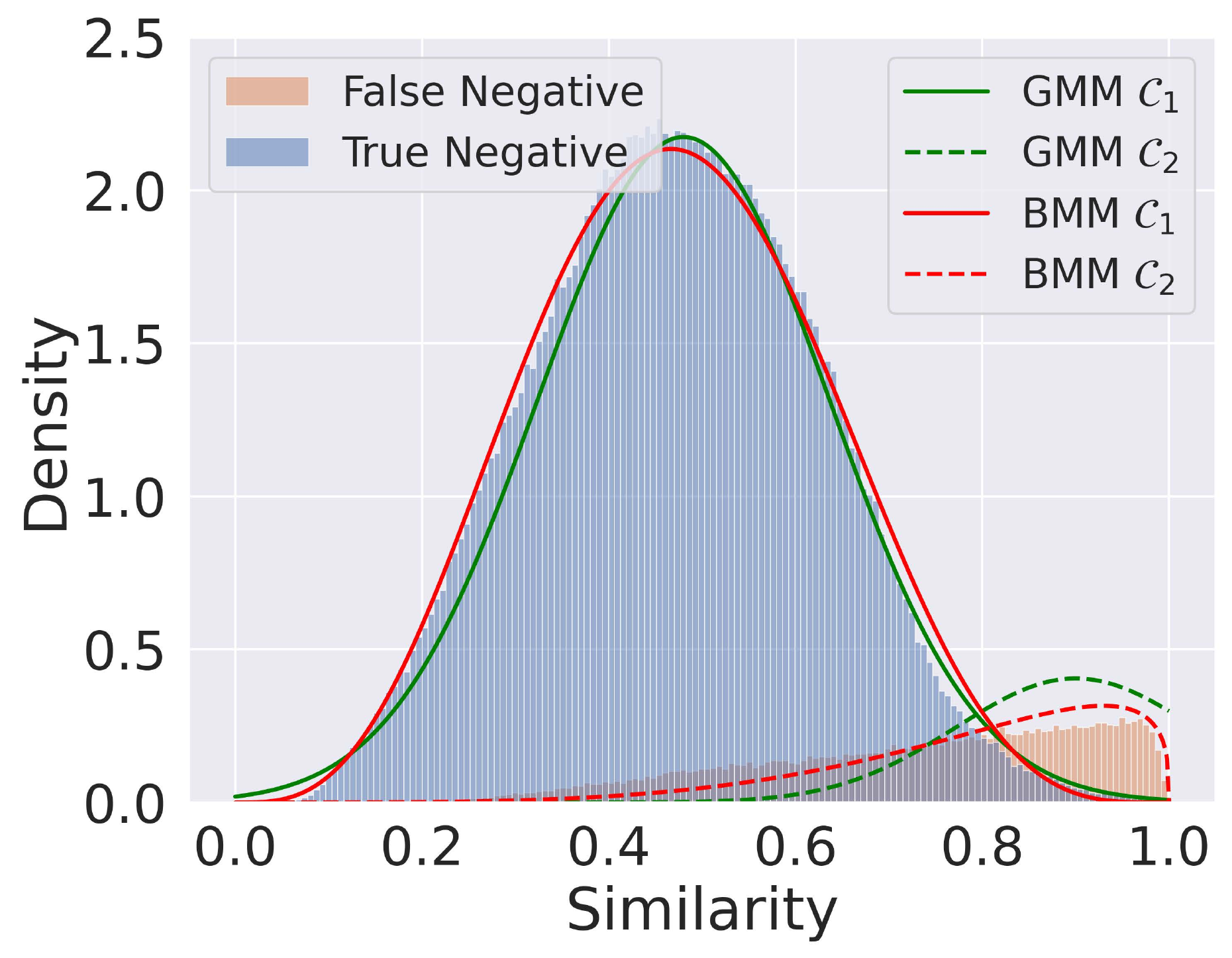}}
    \caption{Empirical distribution and estimated GMM and BMM on Amazon-Photo and Coauthor-CS datasets. The BMM fits the distributions better. Here $\mathcal{C}_1, \mathcal{C}_2$ denote the estimated distributions of two components respectively. Note that BMM is defined on the interval $\left[0,1\right]$ while GMM is defined over $\left[-\infty,+\infty\right]$.}
    \label{fig_3}
\end{figure}
where $\lambda_{c}$ are the mixture coefficients. Here we can fit a two-component BMM (i.e., $C=2$) to model the distribution of true and false negatives. We then utilize Expectation Maximization (EM) algorithm to fit BMM to the observed distribution. In E-step, we fix the parameters of BMM ($\lambda_{c}, \alpha_{c}, \beta_{c}$) and update $p(c\mid s)$ with Bayes rule,
\begin{equation}
p(c\mid s)=\frac{\lambda_{c} p\left(s \mid \alpha_{c}, \beta_{c}\right)}{\sum_{j=1}^{C} \lambda_{j} p\left(s \mid \alpha_{j}, \beta_{j}\right)}.
\end{equation}
In practice, fitting BMM with all the similarities will incur high computational cost. Instead, we can fit BMM well by only sampling $M$ ($M \ll N^2$) similarities for simplification. Larger $M$ bring limited benefits, which we will discuss in the appendix. We can then obtain the weighted average $\bar{s}_{c}$ and variance $v_{c}^{2}$ over $M$ similarities, 
\begin{equation}
\bar{s}_{c}=\frac{\sum_{i=1}^{M} p(c\mid s_{i}) s_{i}}{\sum_{i=1}^{M} p(c\mid s_{i})}, \quad v_{c}^{2}=\frac{\sum_{i=1}^{M} p(c\mid s_{i})\left(s_{i}-\bar{s}_{c}\right)^{2}}{\sum_{i=1}^{M} p(c\mid s_{i})}.
\end{equation}
For M-step, the component model parameters $\alpha_{c}, \beta_{c}$ can be estimated using the method of moments in statistics, 
\begin{equation}
\alpha_{c}=\bar{s}_{c}\left(\frac{\bar{s}_{c}\left(1-\bar{s}_{c}\right)}{v_{c}^{2}}-1\right), \quad \beta_{c}=\frac{\alpha_{c}\left(1-\bar{s}_{c}\right)}{\bar{s}_{c}},
\end{equation}
and coefficients $\lambda_{c}$ for mixing can be calculated as,
\begin{equation}
\lambda_{c}=\frac{1}{M} \sum_{i=1}^{M} p(c\mid s_{i}).
\end{equation}
The above E and M-steps are iterated until convergence or the maximum of iterations $I$ are reached. In our experiments, we set $I = 10$ and we study the influence of $I$ in the appendix. Finally, we can obtain the probability of a negative being true or negative one relative to the anchor with the their similarity $s$,
\begin{equation}
\label{eq11}
p\left(c \mid s\right)=\frac{\lambda_c p\left(s \mid \alpha_{c}, \beta_{c} \right)}{p\left(s\right)}.
\end{equation}
Note that we can regard the fitted distribution with larger $\lambda_c$ as true negatives’ distribution because there are more true negatives than false ones relative to each anchor in multi-class datasets. Also, we can regard the fitted distribution with smaller mean as true negatives’ distribution because the anchor and false negatives share the same class and they are more similar to each other in general. With the posterior probabilities, we devise two schemes to boost the performance of existing GCL as we elaborate below.
\subsection{Scheme 1: ProGCL-weight}
As revealed above, GCL suffers from severe sampling bias which will undermine the performance. To tackle this problem, we propose a novel measure for negatives' hardness considering the hardness and the probability of a negative being true one simultaneously,
\begin{equation}
\label{measure}
w(i,k) = \frac{p\left(c_t \mid s_{ik}\right)s_{ik}}{\frac{1}{N-1}\sum_{j\neq i}[p\left(c_t \mid s_{ij}\right)s_{ij}]},
\end{equation}
where $s_{ik}$ is the similarity between anchor $\boldsymbol{u}_{i}$ and its inter-view negative sample $\boldsymbol{v}_{k}$ and $p\left(c_t \mid s_{ij}\right)$ denotes the probability of $\boldsymbol{v}_{j}$ being true negative relative to anchor $\boldsymbol{u}_{i}$. Note $w(i,k)$ can be utilized to weight both inter-view $\left(\boldsymbol{u}_{i}, \boldsymbol{v}_{k}\right)$ and intra-view $\left(\boldsymbol{u}_{i}, \boldsymbol{u}_{k}\right)$ negative pair,
 \begin{equation}
\begin{aligned}
    &\ell_w\left(\boldsymbol{u}_{i}, \boldsymbol{v}_{i}\right)=\\
    &\log \frac{e^{\frac{\theta\left(\boldsymbol{u}_{i}, \boldsymbol{v}_{i}\right)}{ \tau}}}{\underbrace{e^{\frac{\theta\left(\boldsymbol{u}_{i}, \boldsymbol{v}_{i}\right)}{\tau}}}_{\text{positive pair }}+\underbrace{\sum_{k\neq i}w(i,k)e^{\frac{\theta\left(\boldsymbol{u}_{i}, \boldsymbol{v}_{k}\right)}{\tau}}}_{\text{inter-view negative pairs}}+\underbrace{\sum_{k\neq i}w(i,k)e^{\frac{\theta\left(\boldsymbol{u}_{i}, \boldsymbol{u}_{k}\right)} {\tau}}}_{\text{intra-view negative pairs}}},
\end{aligned}
\end{equation}
we can then define the new overall loss as the average of all the positive pairs,
\begin{equation}
\label{eq14}
\mathcal{J}_w=-\frac{1}{2N} \sum_{i=1}^{N}\left[\ell_w\left(\boldsymbol{u}_{i}, \boldsymbol{v}_{i}\right)+\ell_w\left(\boldsymbol{v}_{i}, \boldsymbol{u}_{i}\right)\right].
\end{equation}
\subsection{Scheme 2: ProGCL-mix}
\begin{figure}[t]
    \begin{center}
    \includegraphics[width=0.49\textwidth]{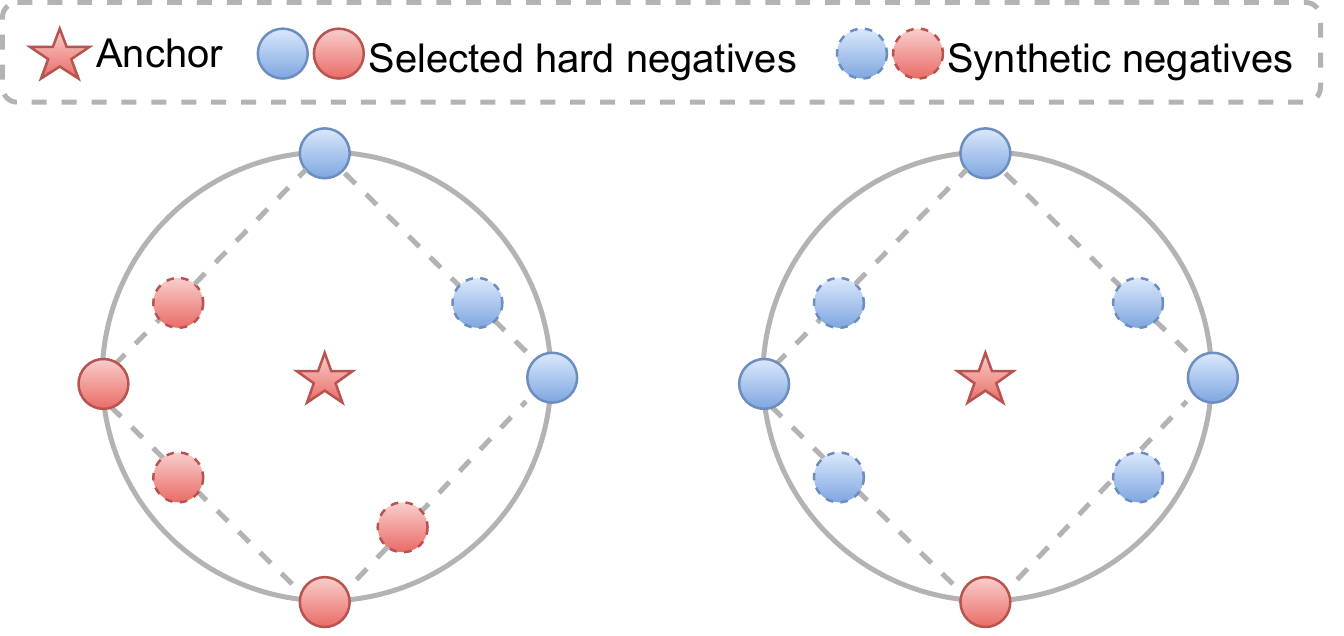}
    \end{center}
    \qquad\quad(a) MoCHi \qquad\qquad\qquad\quad(b) ProGCL-mix
    \caption{Comparison between MoCHi and ProGCL-mix. Gray dotted lines denote mixing and the classes are distinguished by the color of samples. ProGCL-mix synthesizes more true negatives.}
    \label{fig_4}
\end{figure}
Recently, MoCHi~\cite{kalantidis2020hard} proposes to synthesize more negatives with ``hard" negatives selected only with similarity. However, as analysed above, many synthesized hard negatives in GCL are positive samples indeed, which will undermine the performance. To remedy this deficiency, we propose ProGCL-mix which synthesizes more hard negatives considering the probability of a negative being true one. The comparison between MoCHi and ProGCL-mix can be seen in Figure~\ref{fig_4}. More specifically, for each anchor point $\boldsymbol{u}_{i}$, we synthesize $m$ hard negatives by convex linear combinations of pairs of its “hardest” existing negatives. Here, “hardest” existing negatives refers to $N^{\prime}$ negatives that are selected with the measure in Eq. (\ref{measure}). Instead of mixing $N^{\prime}$ samples randomly, we mix them by emphasizing samples that are more likely to be true negative. Formally, for each anchor $\boldsymbol{u}_i$, a synthetic point $\tilde{\boldsymbol{u}}_k$ ($k\in[1,m]$) would be given by,
\begin{small}
\begin{equation}
\label{eq15}
\tilde{\boldsymbol{u}}_k=\alpha_{k} \boldsymbol{v}_{p}+\left(1-\alpha_{k}\right) \boldsymbol{v}_{q},
\end{equation}
where $\boldsymbol{v}_{p},\boldsymbol{v}_{q}$ are selected from $N^{\prime}$ “hardest” existing negatives measured by Eq (\ref{measure}) and $\alpha_{k}$ can be calculated as, 
\begin{equation}
\alpha_{k} = \frac{p\left(c_t \mid s_{ip}\right)}{p\left(c_t \mid s_{ip}\right) + p\left(c_t \mid s_{iq}\right)},
\end{equation}
we can then define the training objective for each positive pair $\left(\boldsymbol{u}_{i}, \boldsymbol{v}_{i}\right)$ with the synthetic negatives,
\begin{equation}
\begin{aligned}
    &\ell_m\left(\boldsymbol{u}_{i},\boldsymbol{v}_{i}\right)=\\
    &\log \frac{e^{\frac{\theta\left(\boldsymbol{u}_{i}, \boldsymbol{v}_{i}\right)}{ \tau}}}{\underbrace{e^{\frac{\theta\left(\boldsymbol{u}_{i}, \boldsymbol{v}_{i}\right)}{\tau}}}_{\text{positive pair }}\!+\!\underbrace{\sum_{k\neq i}e^{\frac{\theta\left(\boldsymbol{u}_{i}, \boldsymbol{v}_{k}\right)}{\tau}}}_{\text{inter-view negative pairs}}\!+\!\underbrace{\sum_{k\neq i}e^{\frac{\theta\left(\boldsymbol{u}_{i}, \boldsymbol{u}_{k}\right)}{\tau}}}_{\text{intra-view negative pairs}}\!+\!\underbrace{\sum_{k=1}^{m}e^{\frac{\theta\left(\boldsymbol{u}_{i}, \tilde{\boldsymbol{u}}_{k}\right)}{\tau}}}_{\text{synthetic negative pairs}}}.
\end{aligned}
\end{equation}
\end{small}

Note that we synthesize new samples only with inter-view hard negatives. Finally, we can define the new overall loss,
\begin{equation}
\label{eq18}
\mathcal{J}_m=-\frac{1}{2N} \sum_{i=1}^{N}\left[\ell_m\left(\boldsymbol{u}_{i}, \boldsymbol{v}_{i}\right)+\ell_m\left(\boldsymbol{v}_{i}, \boldsymbol{u}_{i}\right)\right].
\end{equation}
\subsection{Time Complexity Analysis}
\label{remark}
It is noteworthy that the estimation of the posterior probabilities introduces light computational overhead over the base model. Firstly, we only have to fit BMM once during the training process instead of once per epoch. Secondly, we can fit BMM well with $M$ ($M \ll N^2$) similarities and the time complexity of EM algorithm for fitting in ProGCL is $\mathcal{O}(IM)$. $I$ is the the maximum of iterations. Thirdly, we only have to fit BMM with similarities from single view because both inter-view and intra-view include all negative pairs. In our experiments, we only utilize similarities from inter-view to fit BMM. The training algorithms of both ProGCL-weight and ProGCL-mix for transductive tasks are summarized in Algorithm~\ref{alg1}. The algorithm for inductive tasks can be found in the appendix. 
\begin{algorithm}[ht]
  \caption{ProGCL-weight $\And$ -mix (Transductive)}
  \label{alg1}
\begin{algorithmic}
  \STATE {\bfseries Input:} $\mathcal{T}, \mathcal{G}, f, g, N,$ normalized cosine similarity $s$, epoch for fitting BMM $E$, $mode$ (`weight' or `mix'). 
  \FOR{$epoch = 0,1,2,...$}
      \STATE Draw two augmentation functions $t\sim\mathcal{T}$, $t^{\prime}\sim\mathcal{T}$
      \STATE $\widetilde{\mathcal{G}}_{1}=t(\mathcal{G})$, $\widetilde{\mathcal{G}}_{2}=t^{\prime}(\mathcal{G})$;
      \STATE  $\mathcal{U} = f(\widetilde{\mathcal{G}}_{1})$, $\mathcal{V} = f(\widetilde{\mathcal{G}}_{2})$;
      \FOR{$\boldsymbol{u}_{i}\in\mathcal{U}$ and $\boldsymbol{v}_{i}\in\mathcal{V}$} 
          \STATE$s_{ij}= s(g(\boldsymbol{u}_{i}), g(\boldsymbol{v}_{i}))$
          \IF{$epoch = E$} 
              \STATE Compute $p\left(c_t \mid s_{ij}\right)$ with Eq.~(\ref{eq3}) to Eq.~(\ref{eq11}).
          \ENDIF
    \ENDFOR
    \IF{$epoch \geq E$}
          \IF{$mode$ = `weight'} 
              \STATE Compute $\mathcal{J}_w$ with Eq.~(\ref{measure}) to Eq.~(\ref{eq14}).
              \STATE Update the parameters of $f,g$ with $\mathcal{J}_w$; 
          \ENDIF
            \IF{$mode$ = `mix'}
              \STATE Compute $\mathcal{J}_m$ with Eq.~(\ref{eq15}) to Eq.~(\ref{eq18}).
              \STATE Update the parameters of $f,g$ with $\mathcal{J}_m$.
          \ENDIF
        \ELSE 
            \STATE Compute $\mathcal{J}$ with Eq.~(\ref{eq1}) to Eq.~(\ref{eq2}).
          \STATE Update the parameters of $f,g$ with $\mathcal{J}$.
        \ENDIF
  \ENDFOR
  \STATE {\bfseries Output:} $f,g$.
\end{algorithmic}
\end{algorithm}

\section{Experiments}

\subsection{Experimental Setup $\And$ Baselines}
Following previous work~\cite{velickovic2019deep}, we train the model in an unsupervised manner. The resulting embeddings are utilized to train and test a simple $\ell_2$-regularized logistic classifier. We train the classifier for 20 runs. Additionally, we adopt GRACE as base model and measure performance using micro-averaged F1-score on inductive tasks. For transductive tasks, we adopt GCA as base model and report the test accuracy.\\
\textbf{Transductive learning.} We adopt a two-layer GCN~\cite{kipf2016semi} as the encoder for transductive learning following previous works~\cite{velickovic2019deep,zhu2020deep}. 
We consider our ProGCL with multiple baselines including DeepWalk~\cite{perozzi2014deepwalk} and node2vec~\cite{grover2016node2vec}. Additionally, we also consider recent methods including Graph AutoEncoders (GAE, VGAE)~\cite{kipf2016variational}, DGI, GMI, MVGRL, MERIT, BGRL and GCA as introduced in the related work. We report the best performance of three variants of GCA. We also compare ProGCL with supervised counterparts including GCN~\cite{kipf2016semi} and Graph Attention Networks (GAT)~\cite{velivckovic2017graph}.\\
\textbf{Inductive learning on large graphs.}
Considering the large scale of some graph datasets, we adopt a three-layer GraphSAGE-GCN with residual connections as the encoder following DGI. We adopt the sub-sampling strategy in GraphSAGE where we first select a mini-batch of nodes and then a subgraph centered around each selected node is obtained by sampling node neighbors with replacement. More specifically, we sample 10, 10 and 25 neighbors at the first, second and third level respectively as DGI. The batchsize of our experiments is 256. Also, we estimate the posterior with pairwise similarities among each mini-batch instead of total training set, which we elaborate on in the appendix. We set traditional methods DeepWalk and deep learning based methods unsupervised GraphSAGE, DGI, a recent block-contrastive method COLES~\cite{zhu2021contrastive} and GMI as baselines. To compare ProGCL with supervised counterparts, we report the performance of two supervised methods FastGCN~\cite{chen2018fastgcn} and GraphSAGE. More details can be seen in the appendix.\\
\begin{table*}[t]
\caption{Summary of the accuracies ($\pm$ std) on transductive node classification. The `Available Data' refers to data we can obtain for training, where $\boldsymbol{X}, \boldsymbol{A}$ and $\boldsymbol{Y}$ denotes feature matrix, adjacency matrix and label matrix respectively. `\textbf{OOM}': out of memory on a 32GB GPU. We highlight the performance of ProGCL with gray background. The highest performance of unsupervised models is highlighted in boldface; the highest performance of supervised models is underlined. The baselines marked with '*' are reproduced with the same experimental settings (20 random dataset splits and model initializations). The other results are taken from previously published reports.}
\label{table2}
\setlength{\tabcolsep}{3.6pt}
\centering
\fontsize{8.8pt}{\baselineskip}\selectfont
\begin{tabular}{@{}cccccc@{}}
\toprule
Method  & Available Data & Amazon-Photo & Amazon-Computers &             Coauthor-CS &   Wiki-CS \\ \midrule
Raw features        &  $\boldsymbol{X}$             & 78.53 $\pm$ 0.00             &73.81 $\pm$ 0.00                  &  90.37 $\pm$ 0.00           & 71.98 $\pm$ 0.00        \\
node2vec           &  $\boldsymbol{A}$             &  89.67 $\pm$ 0.12            & 84.39 $\pm$ 0.08                  & 85.08 $\pm$ 0.03            & 71.79 $\pm$ 0.05        \\
DeepWalk           &  $\boldsymbol{A}$             & 89.44 $\pm$ 0.11             & 85.68 $\pm$ 0.06                 & 84.61 $\pm$ 0.22            & 74.35 $\pm$ 0.06        \\
DeepWalk + features &  $\boldsymbol{X},\boldsymbol{A}$              & 90.05 $\pm$ 0.08             & 86.28 $\pm$ 0.07                  & 87.70 $\pm$ 0.04           & 77.21 $\pm$ 0.03        \\ \midrule
GAE                & $\boldsymbol{X},\boldsymbol{A}$              & 91.62 $\pm$ 0.13             & 85.27 $\pm$ 0.19                 & 90.01 $\pm$ 0.17            & 70.15 $\pm$ 0.01        \\
VGAE               &  $\boldsymbol{X},\boldsymbol{A}$              & 92.20 $\pm$ 0.11             & 86.37 $\pm$ 0.21                 &92.11 $\pm$ 0.09             &75.35 $\pm$ 0.14         \\
DGI               &  $\boldsymbol{X},\boldsymbol{A}$              & 91.61 $\pm$ 0.22             & 83.95 $\pm$ 0.47                 & 92.15 $\pm$ 0.63            & 75.35 $\pm$ 0.14        \\
GMI               &  $\boldsymbol{X},\boldsymbol{A}$              & 90.68 $\pm$ 0.17             & 82.21 $\pm$ 0.31                 & \textbf{OOM}            & 74.85 $\pm$ 0.08        \\
MVGRL$^{*}$             &  $\boldsymbol{X},\boldsymbol{A}$              & 92.08 $\pm$ 0.01           & 87.45 $\pm$ 0.21                 & 92.18 $\pm$ 0.15             & 77.43 $\pm$ 0.17        \\
BGRL$^{*}$               &  $\boldsymbol{X},\boldsymbol{A}$              & 92.95 $\pm$ 0.07             &  87.89 $\pm$ 0.10                 & 92.72 $\pm$ 0.03           & 78.41 $\pm$ 0.09       \\ 
MERIT$^{*}$&  $\boldsymbol{X},\boldsymbol{A}$ & 92.53 $\pm$ 0.15 & 88.01 $\pm$ 0.12 &92.51 $\pm$ 0.14 & 78.35 $\pm$ 0.05    \\
GCA$^{*}$                &  $\boldsymbol{X},\boldsymbol{A}$              & 92.55 $\pm$ 0.03            & 87.82 $\pm$ 0.11                & 92.40 $\pm$ 0.07           & 78.26 $\pm$ 0.06       \\ 

\rowcolor{light-gray} 
\textbf{ProGCL-weight}      &  $\boldsymbol{X},\boldsymbol{A}$              &  93.30 $\pm$ 0.09            & 89.28 $\pm$ 0.15                &  93.51 $\pm$ 0.06          &  \textbf{78.68 $\pm$ 0.12}       \\ \rowcolor{light-gray}
\textbf{ProGCL-mix}         &  $\boldsymbol{X},\boldsymbol{A}$              & \textbf{93.64 $\pm$ 0.13}             &  \textbf{89.55 $\pm$ 0.16}               &  \textbf{93.67 $\pm$ 0.12}           & 78.45 $\pm$ 0.04        \\
\midrule\midrule
Supervised GCN                &  $\boldsymbol{X},\boldsymbol{A},\boldsymbol{Y}$             & 92.42 $\pm$ 0.22             & 86.51 $\pm$ 0.54                 & \underline{93.03 $\pm$ 0.31}          & 77.19 $\pm$ 0.12         \\
Supervised GAT                &  $\boldsymbol{X},\boldsymbol{A},\boldsymbol{Y}$             & \underline{92.56 $\pm$ 0.35}             & \underline{86.93 $\pm$ 0.29}                 & 92.31 $\pm$ 0.24            & \underline{77.65 $\pm$ 0.11}        \\ \bottomrule
\end{tabular}
\end{table*}
\subsection{Datasets}
We conduct experiments on seven widely-used datasets including Amazon-Photo, Amazon-Computers, Wiki-CS, Coauthor-CS, Reddit, Flickr and ogbn-arXiv. To keep fair, \emph{for transductive tasks}, we split Amazon-Photo, Amazon-Computers, Wiki-CS and Coauthor-CS for the training, validation and testing following~\cite{zhu2021graph}. \emph{For inductive task}, we split Reddit and Flickr following~\cite{velickovic2019deep,zeng2019graphsaint}. The experimental setting of ogbn-arXiv is the same as BGRL~\cite{thakoor2021bootstrapped}. More information of the datasets is in the appendix.
\begin{table}[t]
\caption{Summary of the micro-averaged $F_1$ scores ($\pm$ std) on inductive node classification.}
\label{table3}
\setlength{\tabcolsep}{3pt}
\centering
\fontsize{8.8pt}{\baselineskip}\selectfont
\begin{tabular}{@{}cccc@{}}
\toprule
Method              & Available Data & Flickr & Reddit \\ \midrule
Raw features        & $\boldsymbol{X}$   & 20.3  & 58.5            \\
DeepWalk            & $\boldsymbol{A}$   & 27.9  & 32.4            \\ \midrule
GraphSAGE       & $\boldsymbol{X},\boldsymbol{A}$     & 36.5           & 90.8            \\
DGI               &  $\boldsymbol{X},\boldsymbol{A}$    & 42.9$\pm$0.1           & 94.0$\pm$0.1            \\

GMI &  $\boldsymbol{X},\boldsymbol{A}$           & 44.5$\pm$0.2   & 94.8$\pm$0.0              \\
COLES-S$^{2}$GC &  $\boldsymbol{X},\boldsymbol{A}$           & 46.8$\pm$0.5   & 95.2$\pm$0.3              \\
GRACE              & $\boldsymbol{X},\boldsymbol{A}$      & 48.0$\pm$0.1        & 94.2$\pm$0.0              \\
\rowcolor{light-gray}
\textbf{ProGCL-weight}      & $\boldsymbol{X},\boldsymbol{A}$     & 49.2$\pm$0.6  & 95.1$\pm$0.2              \\
\rowcolor{light-gray}
\textbf{ProGCL-mix}         & $\boldsymbol{X},\boldsymbol{A}$   &  \textbf{50.0$\pm$0.3}            & \textbf{95.6$\pm$0.1}              \\ \midrule\midrule
Supervised FastGCN             & $\boldsymbol{X},\boldsymbol{A},\boldsymbol{Y}$ &48.1$\pm$0.5  & 89.5$\pm$1.2            \\
Supervised GraphSAGE           & $\boldsymbol{X},\boldsymbol{A},\boldsymbol{Y}$   & \underline{50.1$\pm$1.3} & \underline{92.1$\pm$1.1}            \\ \bottomrule
\end{tabular}
\vspace{-1.5em}
\end{table}
\subsection{Comparison with State-of-the-art Results}
For transductive classification, as can be observed in Table~\ref{table2}, ProGCL consistently performs better than previous unsupervised baselines or even the supervised baselines, which validates the superiority of our ProGCL. We provide more observations as following. Firstly, traditional methods node2vec and DeepWalk only using adjacency matrix outperform the simple logistic regression classifier that only uses raw features (``Raw features'') on Amazon datasets. However, the latter can perform better on Coauthor-CS and Wiki-CS. Combing the both (``DeepWalk + features'') can bring significant improvements. Compared with GCA, our ProGCL emphasize the hard negatives or remove sampling bias, which lifts the representation quality. Secondly, ProGCL-mix performs better than ProGCL-weight in general. For inductive tasks, ProGCL also achieves competitive performance over other baselines as shown in Table~\ref{table3}. 
\begin{table}[t]
\caption{Performance on the ogbn-arXiv measured in accuracy along with standard deviations. The results of baselines are taken from published reports. `$-$' means that the results are unavailable.}
\label{table_large}
\setlength{\tabcolsep}{3pt}
\centering
\fontsize{8.8pt}{\baselineskip}\selectfont
\begin{tabular}{lcc}
\toprule & Validation & Test \\
\hline MLP & 57.65 $\pm$ 0.12 & 55.50 $\pm$ 0.23 \\
node2vec & 71.29 $\pm$ 0.13 & 70.07 $\pm$ 0.13 \\
\hline Random-Init  & 69.90 $\pm$ 0.11 & 68.94 $\pm$ 0.15 \\
DGI  & 71.26 $\pm$ 0.11 & 70.34 $\pm$ 0.16 \\
GRACE-Subsampling  & 72.61 $\pm$ 0.15 & 71.51 $\pm$ 0.11 \\
BGRL  & 72.53 $\pm$ 0.09 & 71.64 $\pm$ 0.12 \\
COLES-S$^{2}$GC & $-$ & 72.48 $\pm$ 0.25\\
\rowcolor{light-gray}\textbf{ProGCL-weight} &72.45 $\pm$ 0.21 & 72.18 $\pm$ 0.09\\
\rowcolor{light-gray}\textbf{ProGCL-mix} &\textbf{72.82 $\pm$ 0.08} &\textbf{72.56 $\pm$ 0.20}\\
\hline
\hline Supervised GCN & 73.00 $\pm$ 0.17 & 71.74 $\pm$ 0.29 \\
\toprule
\end{tabular}
\vspace{-2pt}
\end{table}
\subsection{Results on Large-Scale OGB Dataset}
We conduct experiments on another large graph datasets ogbn-arxiv from OGB benchmark~\cite{hu2020open}. We adopt GRACE with sub-sampling as our base model and a 3-layer GCN as the encoder following~\citet{hu2020open}. As can observed in Table~\ref{table_large}, ProGCL outperforms other unsupervised baselines, which verifies that ProGCL can achieve good  tradeoff between performance and complexity. We report the results on both validation and test sets, as is convention for this task since the dataset is split based on a chronological ordering.
\subsection{Improving Various GCL Methods}
\begin{figure}[ht]
    \begin{center}
    \includegraphics[width=0.398\textwidth]{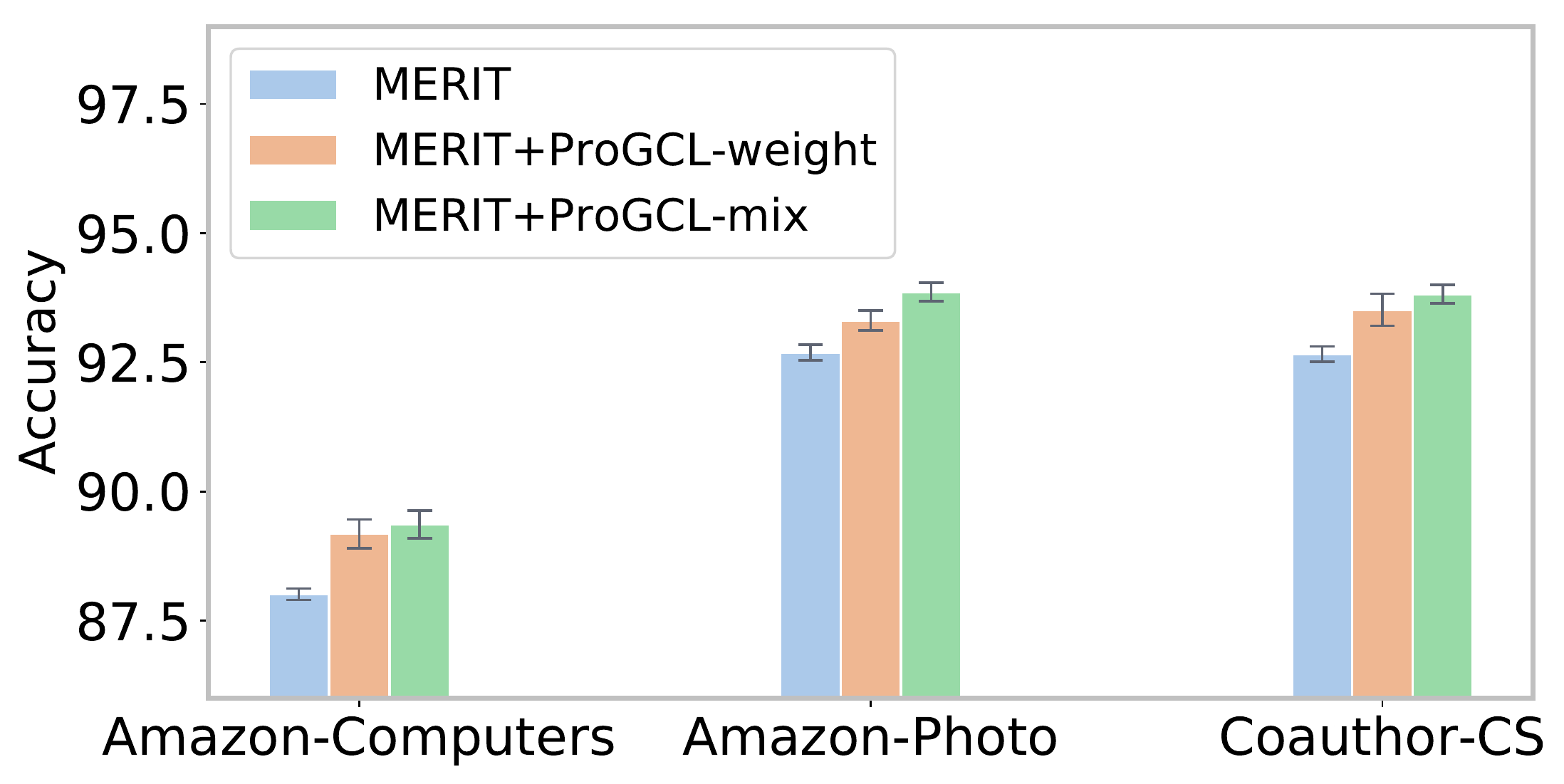}
    \end{center}
    \caption{The performance of ProGCL for another graph contrastive learning method MERIT~\cite{Jin2021MultiScaleCS}.}
    \label{fig_merit}
\end{figure}
In addition to GCA and GRACE, we also evaluate the performance of ProGCL on another GCL method MERIT. The results shown in Figure~\ref{fig_merit} demonstrate that ProGCL brings consistent improvements over the base method, which verifies that our ProGCL is readily pluggable into various negatives-based GCL methods to improve their performance.
\subsection{Why ProGCL Can Alleviate the Bias?}
\begin{figure}[ht]
    \subfigure[Coauthor-CS]{
    \includegraphics[width=0.23\textwidth]{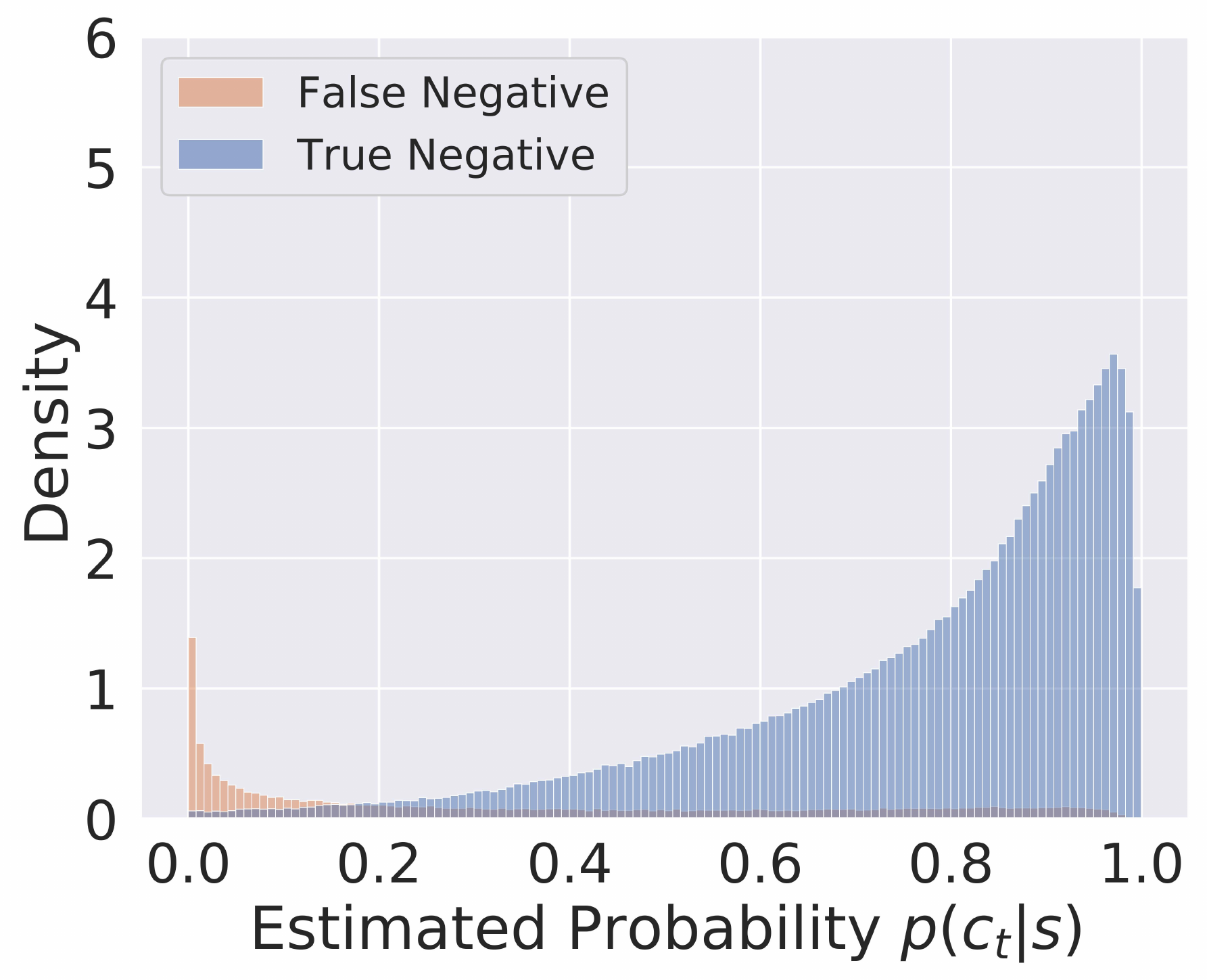}}
    \hspace{-3mm}
    \subfigure[Amazon-Photo]{
    \includegraphics[width=0.23\textwidth]{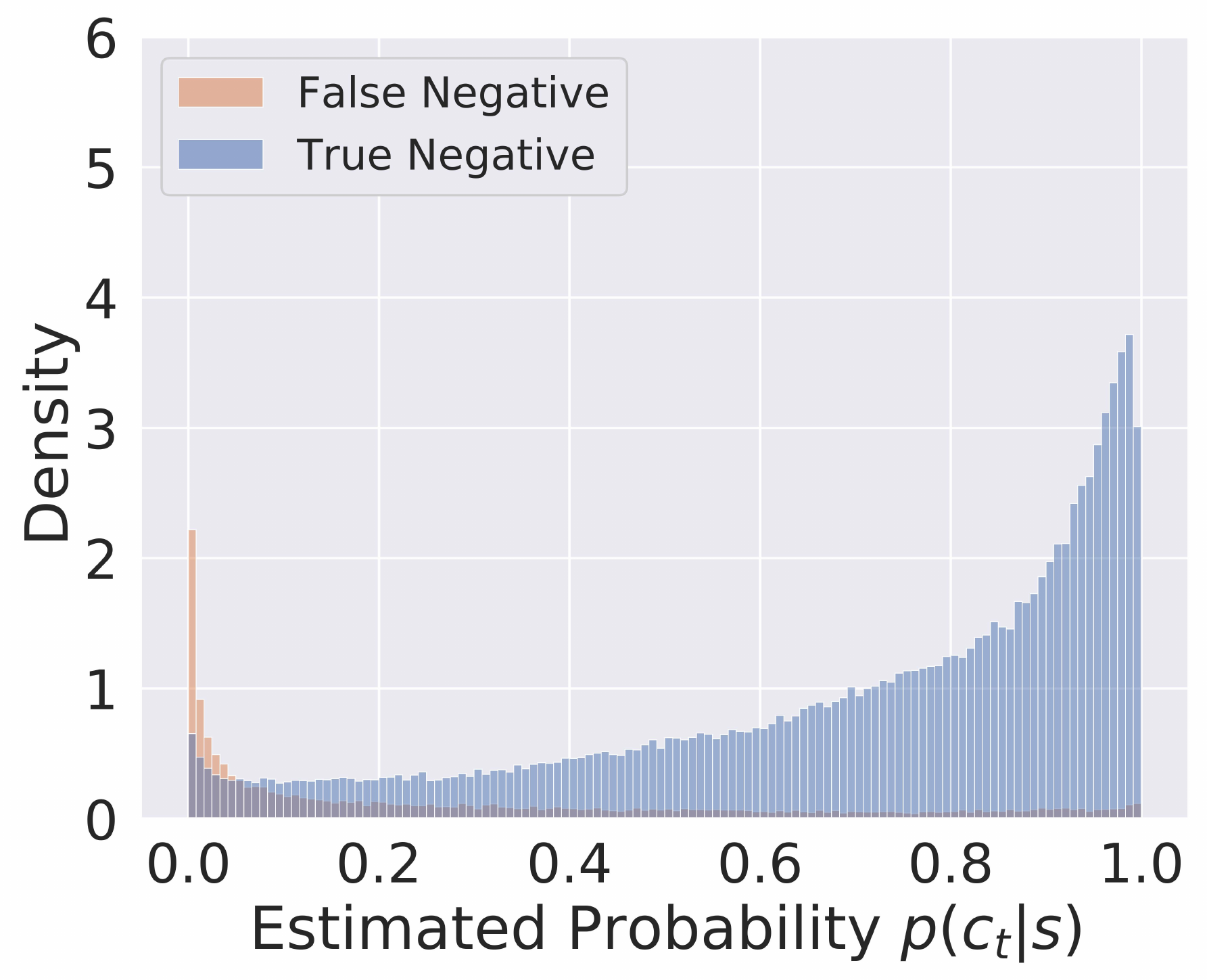}}
    \caption{The histograms of estimated probability.}
    \label{fig_estimation}
\end{figure}
To verify that our ProGCL can alleviate the sampling bias, we plot estimated probability histograms of negatives in Figure~\ref{fig_estimation}. Compared with similarity, the estimated probability serves as a more discriminative measure to distinguish true and false negatives, which can help us select hard and true negatives together with similarity as in Eq. (\ref{measure}).  
\subsection{Ablation Study}
\label{sec:ablation}
In this section, we replace or remove various parts of ProGCL to study the impact of each component.
\begin{table}[ht]
\caption{Comparison between BMM and GMM.}
\label{table4}
\centering
\resizebox{0.48\textwidth}{28.8pt}{
\begin{tabular}{@{}c|cc|cc|cc@{}}
\toprule
Datasets & \multicolumn{2}{c|}{Amazon-Photo} & \multicolumn{2}{c|}{Amazon-Computers} & \multicolumn{2}{c}{Coauthor-CS} \\ \midrule
Scheme   & weight            & mix           & weight              & mix             & weight           & mix          \\ \midrule
GMM      &  92.71                 & 92.83              &  88.35                   & 89.29                & 92.79                 & 92.79             \\
\textbf{BMM}      & \textbf{93.30}                  & \textbf{93.64}            & \textbf{89.28}                   & \textbf{89.55}                & \textbf{93.51}                 & \textbf{93.67}             \\ \bottomrule
\end{tabular}}
\end{table}\\
\textbf{BMM vs. GMM.}  We replace BMM in our ProGCL with GMM and report the performance in Table~\ref{table4}. BMM consistently outperforms GMM in both weight and mix schemes. The reason is that BMM can fit the negatives distribution better than GMM as shown in Figure~\ref{fig_3}.  
\begin{table}[ht]
\caption{Comparison between ProGCL and other hard nagative mining methods. ``$\textcolor{magenta}\uparrow$'' and ``$\textcolor{green}\downarrow$'' refer to performance improvement and drop compared with GCA respectively.}
\label{table5}
\centering
\resizebox{0.48\textwidth}{40pt}{
\begin{tabular}{@{}c|c|c|c@{}}
\toprule
Methods/Datasets  & Amazon-Photo   & Amazon-Computers & Coauthor-CS    \\ \midrule
\multicolumn{1}{l|}{GCA}        & 92.55          & 87.82            & 92.40          \\
+DCL              & 91.02 ($\textcolor{green}\downarrow$ 1.53)          & 86.58 ($\textcolor{green}\downarrow$ 1.24)           & 92.36 ($\textcolor{green}\downarrow$ 0.04) \\
+HCL              & 91.48 ($\textcolor{green}\downarrow$ 1.07)          & 87.21 ($\textcolor{green}\downarrow$ 0.61)           & 93.06 ($\textcolor{magenta}\uparrow$ 0.66) \\
\quad+MoCHi            & 92.36 ($\textcolor{green}\downarrow$ 0.19)          & 87.68 ($\textcolor{magenta}\uparrow$ 0.14)           & 92.58 ($\textcolor{magenta}\uparrow$ 0.18)    \\
+Ring              & 91.33 ($\textcolor{green}\downarrow$ 1.22)          & 84.18 ($\textcolor{green}\downarrow$ 3.64)           & 92.48 ($\textcolor{green}\downarrow$ 0.08) \\

\quad\textbf{+ProGCL-mix} & \textbf{93.64} ($\textcolor{magenta}\uparrow$ 1.09)  & \textbf{89.55} ($\textcolor{magenta}\uparrow$ 1.73)   & \textbf{93.67} ($\textcolor{magenta}\uparrow$ 1.27)\\ \bottomrule
\end{tabular}}
\end{table}\\
\textbf{ProGCL vs. Negative mining techniques.}  To study whether ProGCL can better utilize hard negatives and remove sampling bias in GCL, we equip GCA with DCL~\cite{chuang2020debiased}, HCL~\cite{robinson2021contrastive}, Ring~\cite{wu2021conditional} and MoCHi~\cite{kalantidis2020hard} which have achieved immense success in CL. As shown in Table~\ref{table5}, these techniques bring limited benefits over GCA. Instead, ProGCL introduces consistent and significant improvements over GCA, which further validates that our ProGCL is more suitable for GCL.
\begin{table}[ht]
\caption{Comparison between $p(c_t|s)$ and other alternatives.}
\label{table_prob}
\centering
\resizebox{0.48\textwidth}{32pt}{
\begin{tabular}{@{}c|cc|cc|cc@{}}
\toprule
Datasets & \multicolumn{2}{c|}{Amazon-Photo} & \multicolumn{2}{c|}{Amazon-Computers} & \multicolumn{2}{c}{Coauthor-CS} \\ \midrule
Alternatives   & weight            & mix           & weight              & mix             & weight           & mix          \\ \midrule
$p_r$      &  92.03                 & 92.58              &  87.35                   & 88.85               & 92.06                & 92.71             \\
$p_t$      & 92.49                 & 92.76          & 88.64                  & 89.15                & 92.45                 & 93.22            \\ 
$p(c_t|s)$      & \textbf{93.30}                  & \textbf{93.64}            & \textbf{89.28}                   & \textbf{89.55}                & \textbf{93.51}                 & \textbf{93.67}             \\ \bottomrule
\end{tabular}}
\end{table}\\
\textbf{Performance attributable to ProGCL.} We substitute the estimated probability $p(c_t|s)$ of ProGCL with random probability $p_r$ and tuned $p_t$ (1 - normalized similarity). As shown in Table~\ref{table_prob}, $p_t$ can bring minor improvement over random $p_r$ while the estimated posterior $p(c_t|s)$ by BMM is the best among the three alternatives.
\subsection{Hyperparameters Sensitivity Analysis}
\begin{figure}[ht]
    \begin{center}
    \includegraphics[width=0.326\textwidth]{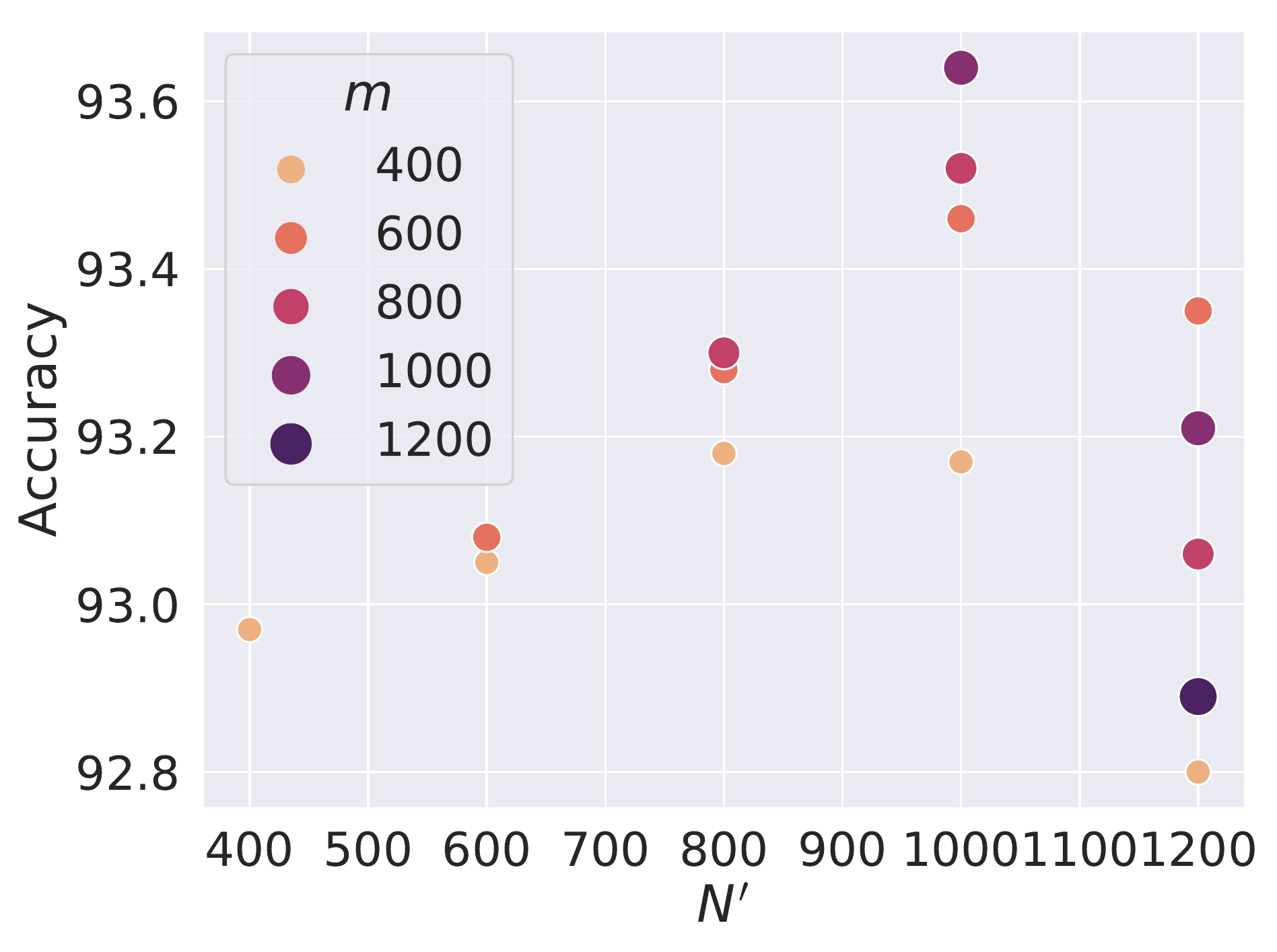}
    \setlength{\belowcaptionskip}{-2.58mm}
    \end{center}
    \caption{Accuracy when varying $N^{\prime}$ (x-axis) and $m$ on Amazon-Photo. We study more hyperparameters in the appendix.}
    \label{fig_7}
\end{figure}
Here we study the number of most hard negatives $N^{\prime}$ and synthetic negatives $m$ of ProGCL-mix. The results shown in Figure~\ref{fig_7} illustrate that more samples mixing bring consistent performance gains in general. However, oversized $N^{\prime}$ demonstrated no significant advantages in both accuracy and efficiency. 
\section{Conclusions}
In this paper, we explain why existing hard negative mining methods can not work well in GCL and contrapuntally introduce BMM to estimate the probability of a negative being true one. Also, we devise two schemes to further boost GCL. Interesting directions of future work include (1) applying GCL to more real-world tasks including social analysis and drug discovery~\cite{sun2021mocl,xia2022towards}; (2) exploring the theoretical explanations for the immense success of contrastive learning. 
\section{Acknowledgements}
This work is supported in part by the Science and Technology Innovation 2030 - Major Project (No. 2021ZD0150100) and National Natural Science Foundation of China (No. U21A20427).
\bibliography{example_paper}
\bibliographystyle{icml2022}

\newpage
\appendix
\onecolumn
\icmltitle{Appendix of ProGCL}
\section{Datasets}
\label{Appendix_A}
\begin{table*}[ht]
\caption{Statistics of datasets used in experiments.}
\label{table7}
\centering
\resizebox{0.688\textwidth}{55.8pt}{
\begin{tabular}{@{}cccccc@{}}
\toprule
\textbf{Dataset} & \textbf{Task} & \textbf{Nodes} & \textbf{Edges}  & \textbf{Features} & \textbf{Classes}  \\ \midrule
\textbf{Amazon-Photo}     & Transductive & 7,650    & 119,081   & 745        & 8      \\
\textbf{Amazon-Computers} & Transductive & 13,752   & 245,861   & 767        & 10     \\
\textbf{Coauthor-CS}      & Transductive & 18,333   & 81,894    & 6,805       & 15   \\
\textbf{Wiki-CS}          & Transductive & 11,701   & 216,123   & 300        & 10    \\ \midrule
\textbf{Flickr}           & Inductive    & 89,250   & 899,756   & 500        & 7   
\\ 
\textbf{Reddit}           & Inductive    & 231,443  & 11,606,919 & 602        & 41  \\
\hline
\textbf{Ogbn-arXiv} & Inductive & 169,343 & 1,166,243 & 128 &40\\\bottomrule
\end{tabular}}
\end{table*}
We introduce the datasets used in our experiments as follows:
\begin{itemize}
\item \textbf{WikiCS}~\cite{mernyei2020wiki} is a reference network constructed based on Wikipedia. The nodes correspond to articles about computer
science and edges are hyperlinks between the articles. Nodes
are labeled with ten classes each representing a branch of
the field. Node features are calculated as the average of pre-trained GloVe~\cite{pennington2014glove} word embeddings of words in each article.
\item \textbf{Amazon-Computers} and \textbf{Amazon-Photo}~\cite{shchur2018pitfalls} are two networks of co-purchase relationships constructed from Amazon, where nodes are goods and two goods are connected
when they are frequently bought together. Each node has a
sparse bag-of-words feature encoding product reviews and
is labeled with its category.
\item \textbf{Coauthor-CS}~\cite{shchur2018pitfalls} is an academic networks, which contain co-authorship graphs based
on the Microsoft Academic Graph from the KDD Cup 2016
challenge. In the graph, nodes represent authors and
edges indicate co-authorship relationships; that is, two nodes
are connected if they have co-authored a paper. Each node
has a sparse bag-of-words feature based on paper keywords
of the author. The label of an author corresponds to their
most active research field.
\item \textbf{Flickr} originates from NUS-wide\footnote{https://lms.comp.nus.edu.sg/research/NUS-WIDE.htm}. The SNAP website\footnote{https://snap.stanford.edu/data/web-flickr.html} collected Flickr data from
four different sources including NUS-wide, and generated an un-directed graph. One node in the
graph represents one image uploaded to Flickr. If two images share some common properties (e.g.,
same geographic location, same gallery, comments by the same user, etc.), there is an edge between
the nodes of these two images. The node features are the 500-dimensional bag-of-word
representation of the images provided by NUS-wide. For labels, we scan over the 81 tags of each
image and manually merged them to 7 classes. Each image belongs to one of the 7 classes.
\item \textbf{Reddit}~\cite{hamilton2017representation} contains Reddit posts created
in September 2014, where posts belong to different communities (subreddit). In the dataset, nodes correspond to posts, and edges connect posts if the same user has commented on both. Node features are constructed from post title, content, and comments, using off-the-shelf GloVe word embeddings, along with other metrics such as post score and the number of comments.
\item \textbf{Ogbn-arXiv} is a large-scale graph from the Open Graph Benchmark~\cite{hu2020open}. This is another citation network, where nodes represent CS papers on arXiv indexed by the Microsoft Academic Graph~\cite{sinha2015overview}. In our experiments, we symmetrize this graph and thus there is an edge between any pair of nodes if one paper has cited the other. Papers are classified into 40 classes based on arXiv subject area. The node features are computed as the average word-embedding of all words in the paper, where the embeddings are computed using a skip-gram model~\cite{mikolov2013distributed} over the entire corpus.
\end{itemize}
\section{More Similarity Histograms of Negativess in SimCLR (CL) and GCA (GCL).}
\begin{figure}[ht]
\centering  
    \subfigure[Amazon-photo (Graph)]{
    \label{fig2-b}
    \includegraphics[width=0.288\textwidth]{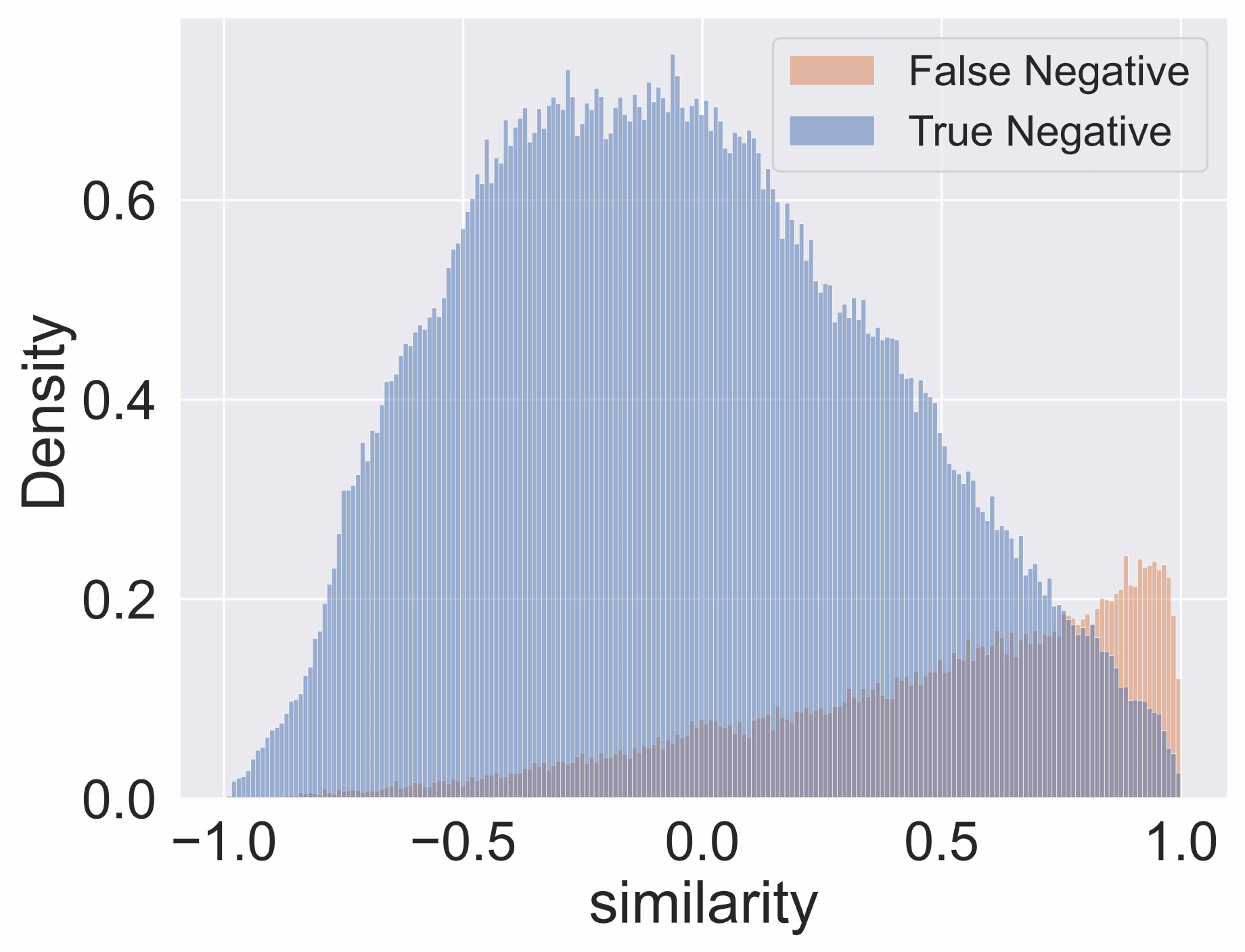}
    }
    \subfigure[Amazon-computers (Graph)]{
    \label{fig2-c}
    \includegraphics[width=0.288\textwidth]{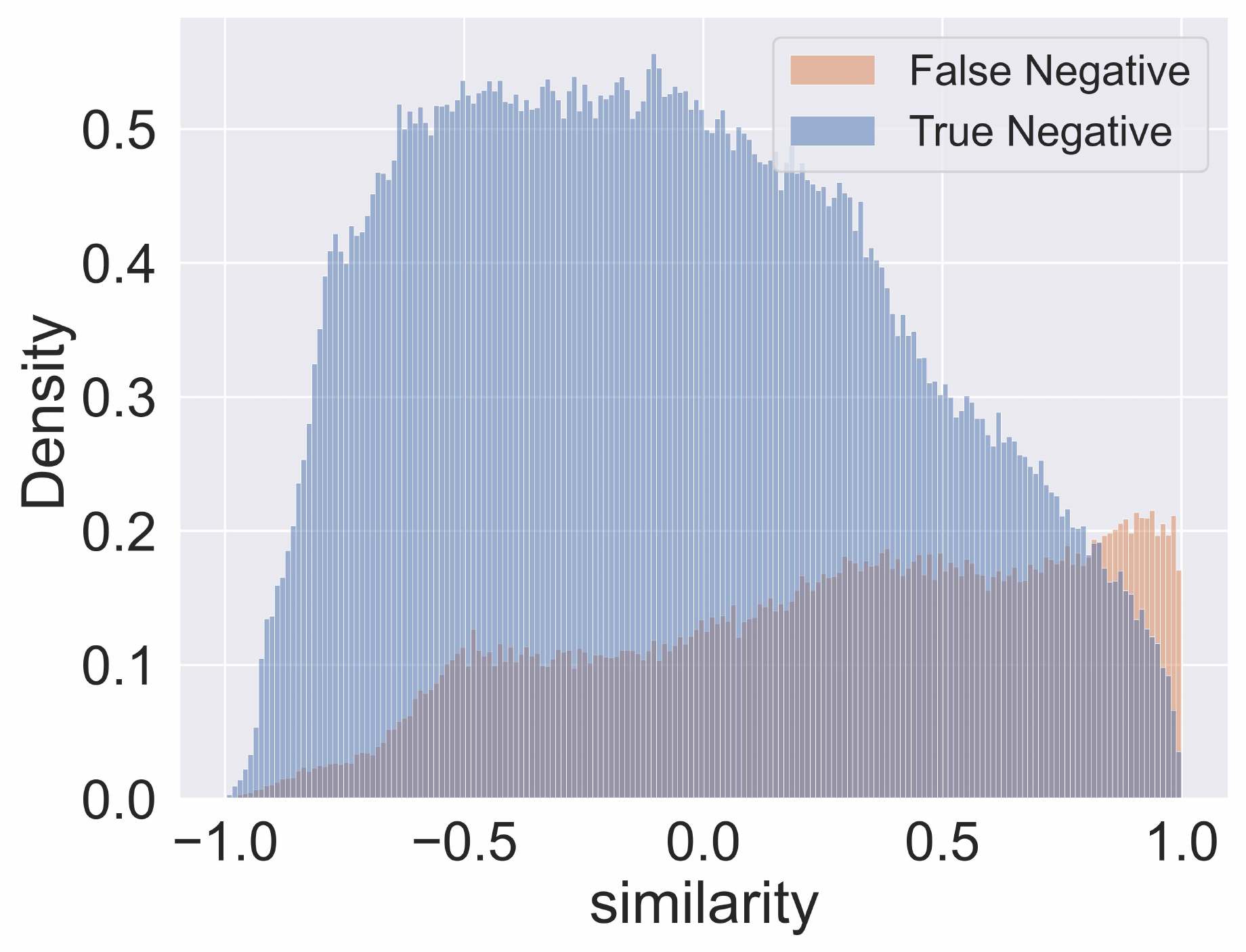}
    }
    \subfigure[WikiCS (Graph)]{
    \label{fig2-d}
    \includegraphics[width=0.288\textwidth]{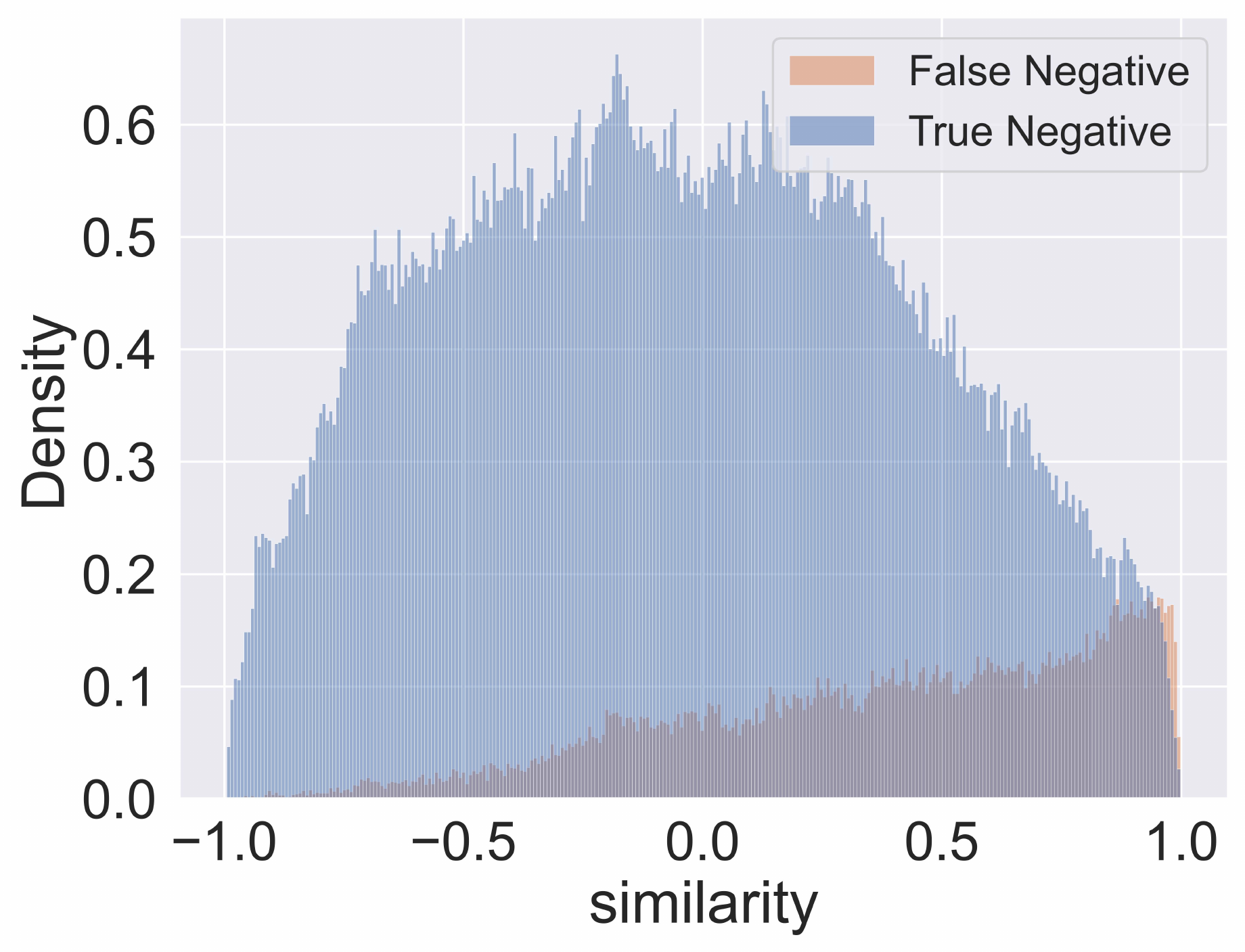}
    }
    \caption{More similarity histograms of negative samples in and GCL.}
    \label{fig2}
\end{figure}
As demonstrated in Figure~\ref{fig2}, the negatives distribution over similarity varies significantly across GCL and CL. These phenomena provide explainations for the failure of existing techniques that emphasize hard negatives in GCL. More specifically, they regard the negatives that are most similar to anchor points as hard ones, which is feasible in CL. However, for graph-structured data (see Figure~\ref{fig2-b}, Figure~\ref{fig2-c} and Figure~\ref{fig2-d}), many “hard” ones are false negatives indeed, which will undesirably push away the semantically similar samples.
\section{More Experimental Details}
\subsection{Transductive learning}
We adopt a two-layer GCN~\cite{kipf2016semi} as the encoder for transductive learning following previous works~\cite{velickovic2019deep,zhu2020deep}. We can describe the architecture of the encoder as,
\begin{equation}
\begin{aligned}
\mathrm{GC}_{i}(\boldsymbol{X}, \boldsymbol{A}) &=\sigma\left(\boldsymbol{\hat{D}}^{-\frac{1}{2}} \boldsymbol{\hat{A}} \boldsymbol{\hat{D}}^{-\frac{1}{2}} \boldsymbol{X} \boldsymbol{W_{i}}\right), \\
f(\boldsymbol{X}, \boldsymbol{A}) &=\mathrm{GC}_{2}\left(\mathrm{GC}_{1}(\boldsymbol{X}, \boldsymbol{A}), \boldsymbol{A}\right),
\end{aligned}
\end{equation}
where $\boldsymbol{\hat{A}}=\boldsymbol{A}+\boldsymbol{I}$ is the adjacency matrix with self-loops and $\boldsymbol{\hat{D}}=\sum_{i} \boldsymbol{\hat{A}}_{i}$ is the degree matrix, $\sigma(\cdot)$ is a nonlinear activation function. $\boldsymbol{W_{i}}$ is the learnable weight matrix.
\subsection{Inductive learning}
Considering the large scale of some graph datasets, we adopt a three-layer GraphSAGE-GCN with residual connections~\cite{he2016deep} as the encoder following DGI~\cite{velickovic2019deep} and GRACE~\cite{zhu2020deep}. The architecture of the encoder can be formulated as,
\begin{equation}
\begin{aligned}
\widehat{\mathrm{MP}}_{i}(\boldsymbol{X}, \boldsymbol{A}) &=\sigma\left(\left[\hat{\boldsymbol{D}}^{-1} \hat{\boldsymbol{A}} \boldsymbol{X} ; \boldsymbol{X}\right] \boldsymbol{W}_{i}\right),\\
f(\boldsymbol{X}, \boldsymbol{A}) &=\widehat{\mathrm{MP}}_{3}\left(\widehat{\mathrm{MP}}_{2}\left(\widehat{\mathrm{MP}}_{1}(\boldsymbol{X}, \boldsymbol{A}), \boldsymbol{A}\right), \boldsymbol{A}\right).
\end{aligned}
\end{equation}

\section{Hyper-parameters Analysis}
\begin{table*}[ht]
\caption{Hyperparameters specifications. ``Starting epoch'' $E$ refers to the epoch for fitting BMM; Initial weight $w_{init}$ is the initial weight for false negative component of mixture distribution; ``Iterations'' is the iterations $I$ of EM algorithm.}
\label{table1}
\centering
\resizebox{0.86\textwidth}{58.8pt}{
\begin{tabular}{ccccccccc}
\toprule
Dataset      & $\tau$ & \begin{tabular}[c]{@{}c@{}}learning\\ rate\end{tabular} & \begin{tabular}[c]{@{}c@{}}Training\\ epochs\end{tabular} & \begin{tabular}[c]{@{}c@{}}Hidden\\ dimension\end{tabular} & \begin{tabular}[c]{@{}c@{}}Activation\\ function\end{tabular} & \begin{tabular}[c]{@{}c@{}}Starting\\ epoch\end{tabular} & \begin{tabular}[c]{@{}c@{}}Initial\\ weight\end{tabular} & Iterations \\ \hline
Amazon-Photo & 0.3 & 0.01                                                    & 2500                                                      & 128                                                        & RRelu                                                         & 400                                                      & 0.15                                                     & 10         \\
Amazon-Computers    & 0.2 & 0.01                                                    & 2000                                                      & 128                                                        & RRelu                                                         & 400                                                      & 0.05                                                     & 10         \\
Coauthor-CS  & 0.2 & 0.0001                                                  & 1000                                                      & 256                                                        & RRelu                                                         & 400                                                      & 0.05                                                     & 10         \\
Wiki-CS      & 0.4 & 0.01                                                    & 4000                                                      & 256                                                        & PRelu                                                         & 50                                                       & 0.05                                                     & 10         \\ \hline
Reddit       & 0.4 & 0.0001                                                 & 80                                                        & 512                                                        & ELU                                                           & 40                                                       & 0.01                                                     & 10         \\
Filckr       & 0.4 & 0.0001                                                   & 80                                                       & 512                                                       & ELU                                                         & 20                                                      & 0.15                                                     & 10         \\ \hline
Ogbn-arXiv      & 0.4 & 0.0001                                                   & 100                                                     & 512                                                       & ELU                                                         & 20                                                      & 0.02                                                    & 10         \\ 

\bottomrule
\end{tabular}}
\end{table*}
\begin{figure*}[ht]
\centering  
    \subfigure[$E$]{
    \label{fig-a}
    \includegraphics[width=0.218\textwidth]{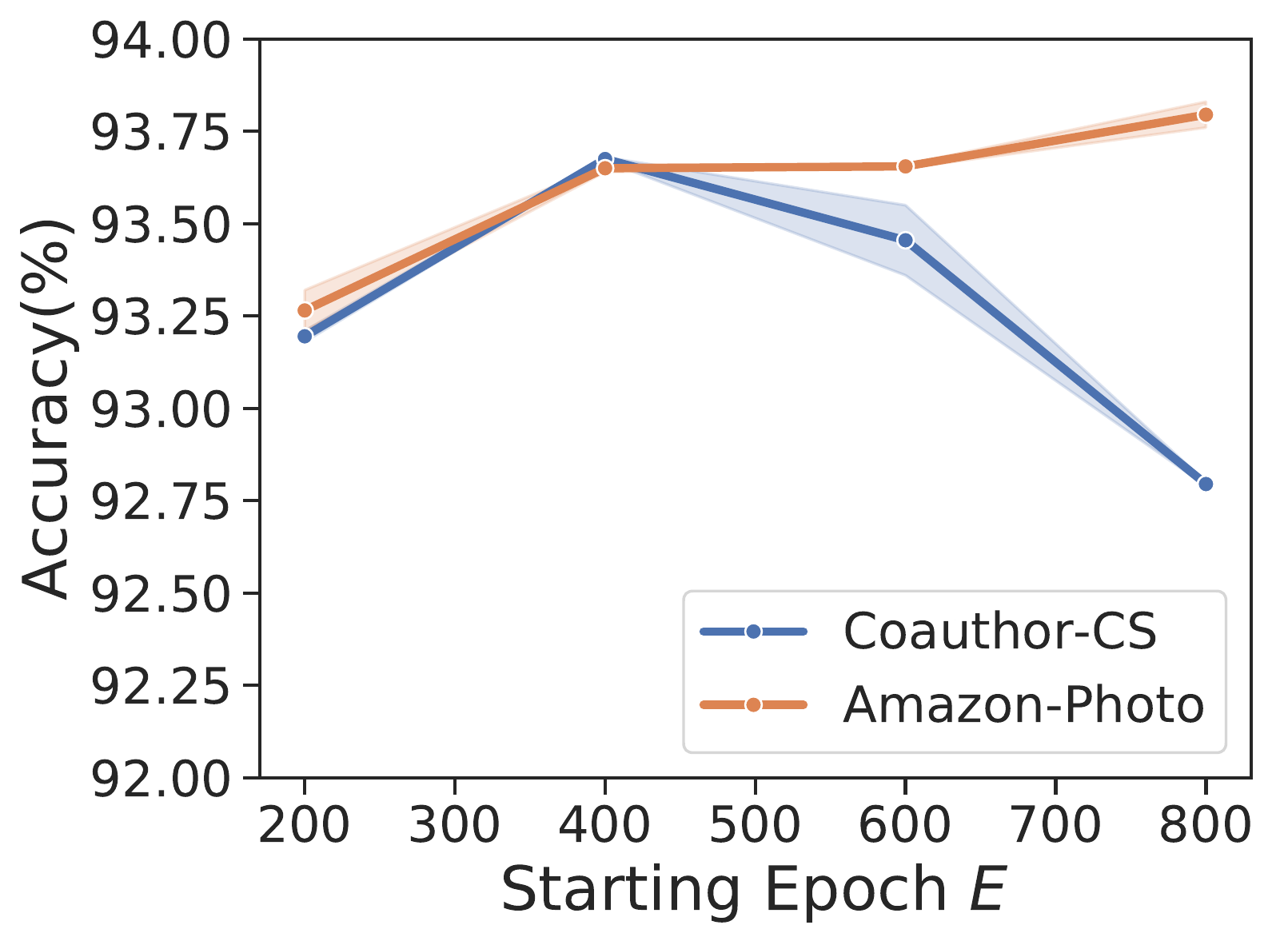}
    }
    \subfigure[$w_{init}$]{
    \label{fig-b}
    \includegraphics[width=0.218\textwidth]{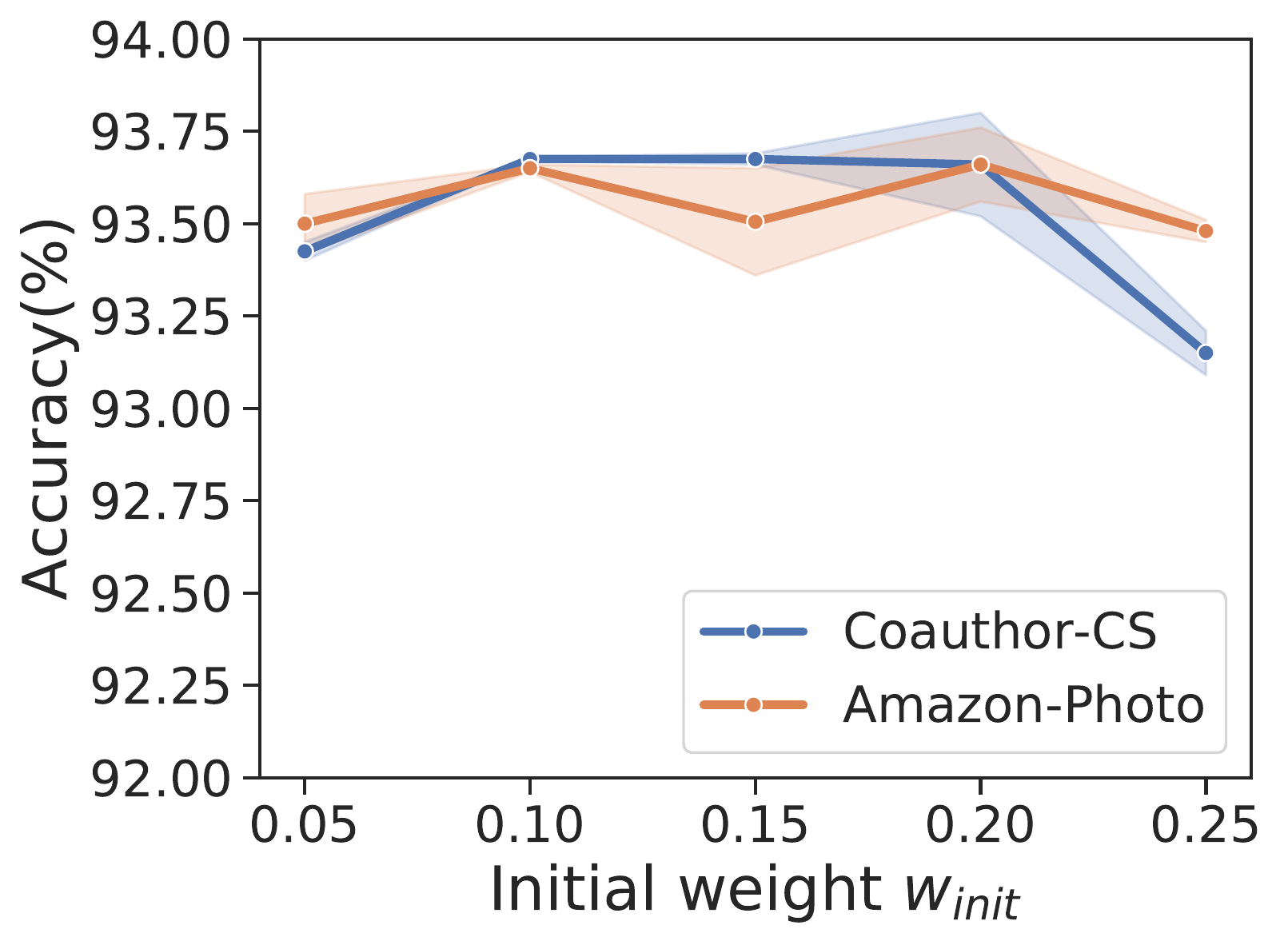}
    }
    \subfigure[$I$]{
    \label{fig-c}
    \includegraphics[width=0.218\textwidth]{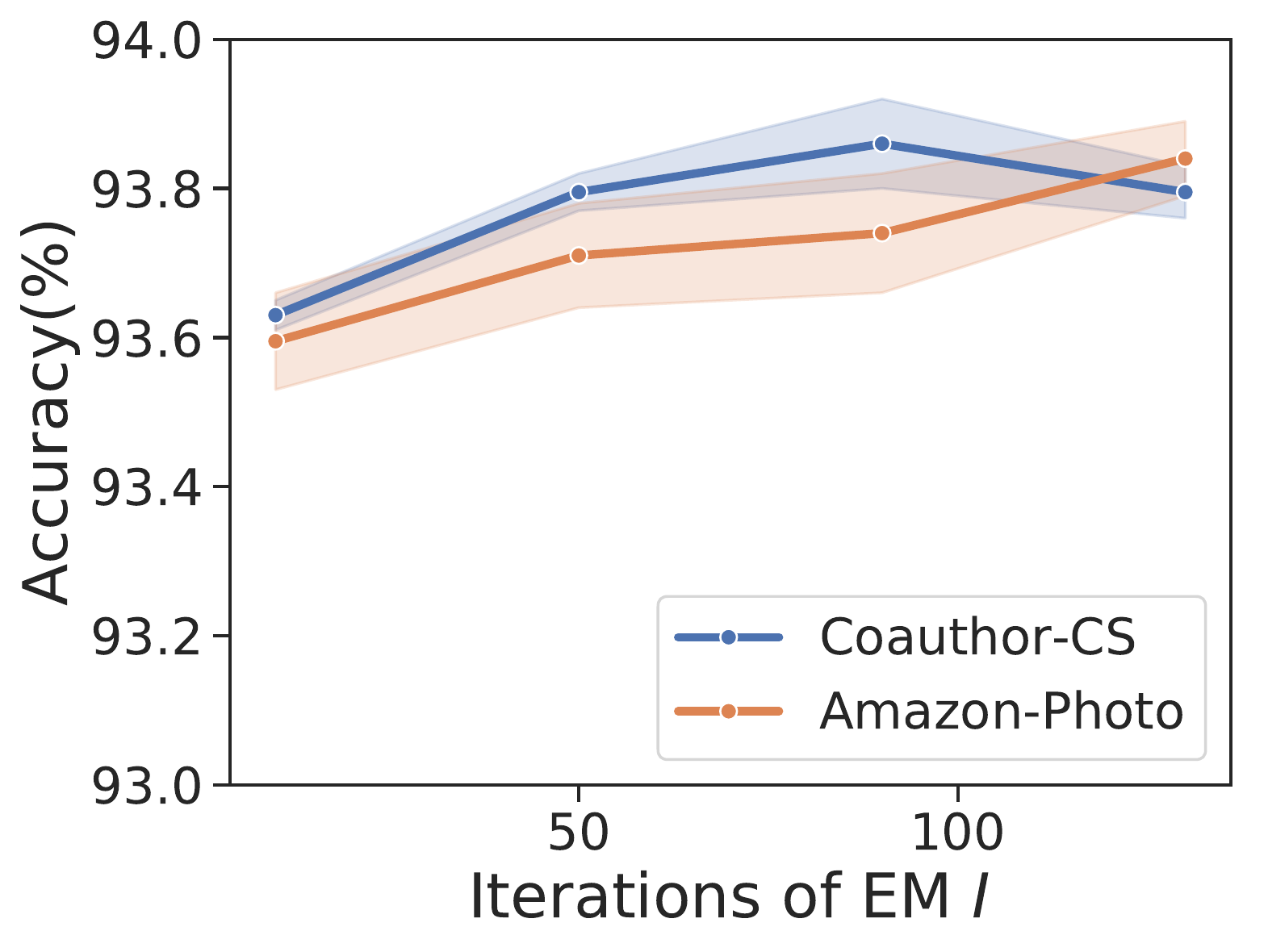}
    }
    \subfigure[$M^{\prime}$]{
    \label{fig-d}
    \includegraphics[width=0.218\textwidth]{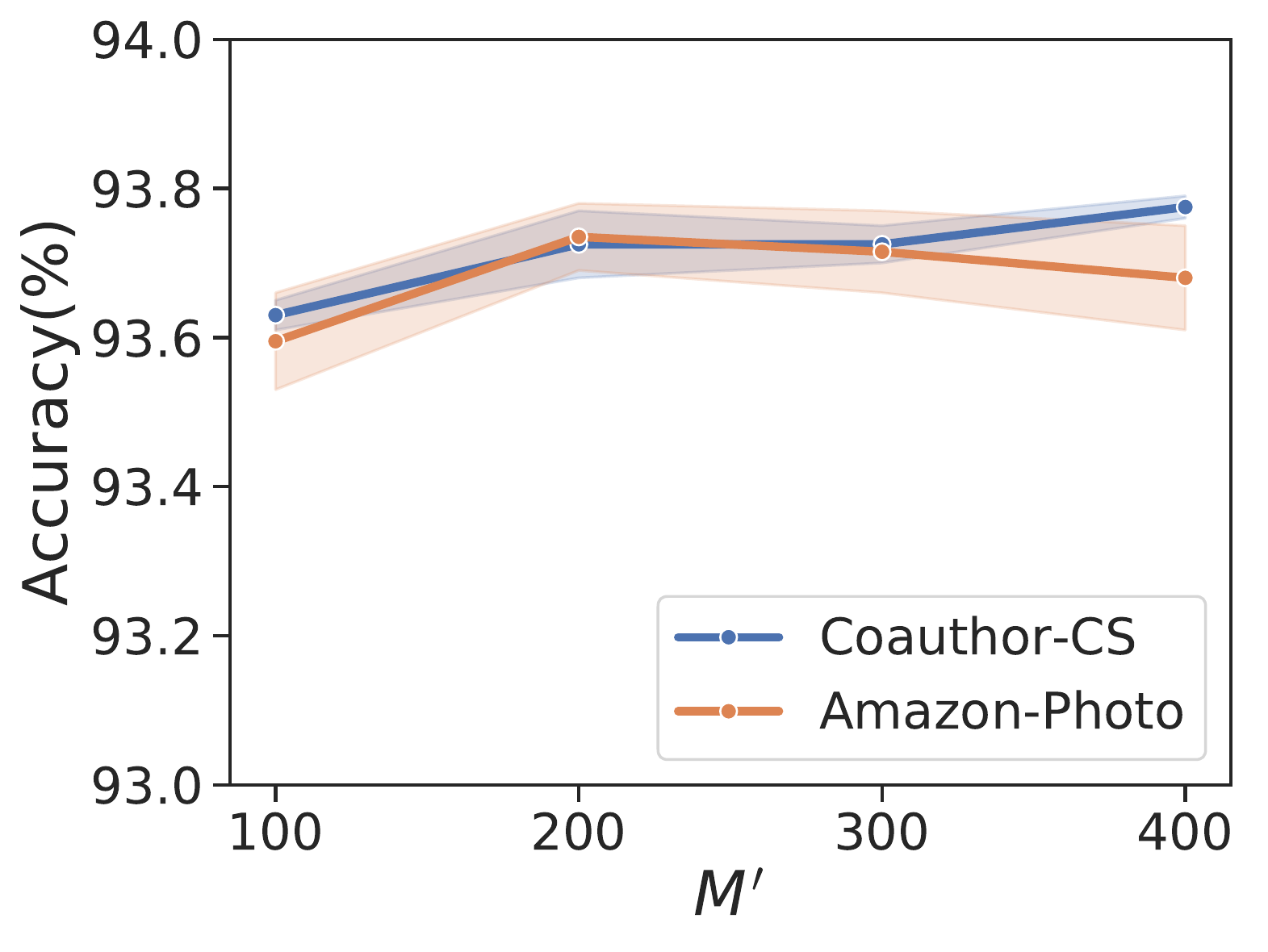}
    }
    \caption{Accuracy when varying $E$, $w_{init}$, $I$ and $M^{\prime} (M = NM^{\prime})$ for BMM.
    }
    \label{fig1}
\end{figure*}
Following GRACE~\cite{zhu2020deep} and GCA~\cite{zhu2021graph}, we initialize the model parameters with Glorot initialization~\cite{glorot2010understanding} and train the model using Adam SGD optimizer for all datasets. The $\ell_2$ weight decay factor is seted as $10^{-5}$ and the dropout rate is seted to zero. The parameters which control the sampling process of two views are the same as GRACE and GCA. Other hyper-parameters of ProGCL can be seen in Table~\ref{table1}. For transductive task, the two hyperparameters were chosen in a grid $E \in \{50, 100, 200, 300, 400, 600, 800\}$ and $w_{init}\in\{0.01, 0.05, 0.10, 0.15, 0.20, 0.25\}$. We further study the influence of $E$, $w_{init}$, $I$ and $M$ in Figure~\ref{fig1}. Firstly, the accuracy varies significantly across various starting epoch $E$ in Amazon-Photo while varies slightly in Coauthor-CS, which illustrates the importance of tuning $E$ varies across datasets. However, generally speaking, ProGCL's performance does not see sharp drop when varying $E$. This validates that BMM is flexible enough to fit various distributions of different epochs. As shown in figure~\ref{fig-b}, the initial weight $w_{init}$ does not influence much, which illustrates that BMM can split the two component well to an ideal ratio even if the initial weight is far from intuitive one (the reciprocal of classes number). As shown in figure~\ref{fig-c}, we can observe minor improvement when we iterate EM algorithm more times. However, this will introduce more computational overhead and thus we set $I=10$ in all experiments for convenience. As shown in Figure~\ref{fig2-d}, sampling more similarities for fitting the BMM can bring minor improvements, however, it will introduce much more computational overhead. In our experiments, we only sample $M^{\prime} = 100$ samples for each anchor point. Thus, the number of total selected samples $M = NM^{\prime}$. 
\section{Pseudo Codes of Inductive Learning.}
Different from transductive learning, it is not feasible to compute all the pairwise similarities for large-scale graph datasets. Thus, we extend ProGCL to inductive setting. The algorithm of ProGCL for inductive learning can be seen as follows. The equations mentioned in the algorithm can be seen in the main text.
\begin{algorithm}[ht]
\caption{ProGCL-weight $\And$ ProGCL-mix (Inductive)}\label{algorithm}
\begin{algorithmic}
\STATE {\bfseries Input:}$\mathcal{T}, \mathcal{G}, f, g, N,$ normalized cosine similarity $s$, epoch for fitting BMM $E$, $mode$ (`weight' or `mix'), empty list $\mathcal{P}$ for storing the estimated probabilities, Batchsize $B$. 
\FOR {$epoch = 0,1,2,...$}
    \STATE $k = 0$;\\
    \FOR {each mini-batch}
        \STATE $\mathcal{G}_{s}\left(\mathcal{V}_{s}, \mathcal{E}_{s}\right) \leftarrow$ Sampled sub-graph of $\mathcal{G}$ with sampling rules of GraphSAGE;\\
        \STATE Draw two augmentation functions $t\sim\mathcal{T}$, $t^{\prime}\sim\mathcal{T}$;\\
        \STATE $\widetilde{\mathcal{G}}_{s}^{(1)}=t(\mathcal{G}_s)$, $\widetilde{\mathcal{G}}_{s}^{(2)}=t^{\prime}(\mathcal{G}_s)$;\\
        \STATE $\mathcal{U}_s = f(\widetilde{\mathcal{G}}_{s}^{(1)})$, $\mathcal{V}_s = f(\widetilde{\mathcal{G}}_{s}^{(2)})$;\\
        \FOR{all $\boldsymbol{u}_{i}\in\mathcal{U}_s$ and $\boldsymbol{v}_{i}\in\mathcal{V}_s$}
            \STATE $s_{ij}= s(g(\boldsymbol{u}_{i}), g(\boldsymbol{v}_{i}))$;\\
            \IF{$epoch = E$}
                \STATE Compute $\mathcal{M}_{i,j} = p\left(c_t \mid s_{ij}\right)$ with Eq.~(\ref{eq3}) to Eq.~(\ref{eq11});\\
                \STATE $\mathcal{P}$.append($\mathcal{M}$).
            \ENDIF\\
            \IF{$epoch \geq E$}
                \IF{$mode$ = 'weight'}
                    \STATE Compute $\mathcal{J}_w$ with Eq.~(\ref{measure}) to Eq.~(\ref{eq14}); ($\mathcal{P}[$k$]$ as the estimated probabilities for $k$-th minibatch)\\
                    \STATE Update the parameters of $f,g$ with $\mathcal{J}_w$.\\
                \ENDIF
                \IF{$mode$ = 'mix'}
                    \STATE Compute $\mathcal{J}_m$ with Eq.~(\ref{eq15}) to Eq.~(\ref{eq18}); ($\mathcal{P}[$k$]$ as the estimated probabilities for $k$-th minibatch)\\
                    \STATE Update the parameters of $f,g$ with $\mathcal{J}_m$.
                \ENDIF
            \ELSIF{$epoch \textless E$}
                \STATE Compute $\mathcal{J}$ with Eq.~(\ref{eq1}) to Eq.~(\ref{eq2});\\
                \STATE Update the parameters of $f,g$ with  $\mathcal{J}$.\\
            \ENDIF
        \ENDFOR   
        \STATE $k = k + 1$.
    \ENDFOR
\ENDFOR
\STATE {\bfseries Output:} $f,g$.
\end{algorithmic}
\end{algorithm}

\end{document}